\documentclass[11pt]{article}

\usepackage[T1]{fontenc}
\usepackage[utf8]{inputenc}
\usepackage{lmodern}
\usepackage[margin=1in]{geometry}
\usepackage{amsmath,amssymb,amsfonts,amsthm,mathtools,bm}
\usepackage{booktabs,longtable,array,graphicx,float,calc,multirow}
\usepackage{enumitem}
\usepackage{algorithm}
\usepackage{algpseudocode}
\makeatletter
\renewcommand{\theHALG@line}{\thealgorithm.\arabic{ALG@line}}
\makeatother
\usepackage{microtype}
\usepackage[authoryear]{natbib}
\usepackage{hyperref}
\hypersetup{colorlinks=true,linkcolor=black,citecolor=black,urlcolor=blue}
\usepackage[nameinlink,capitalise]{cleveref}
\newtheorem{theorem}{Theorem}
\newtheorem{proposition}[theorem]{Proposition}
\newtheorem{corollary}[theorem]{Corollary}
\theoremstyle{definition}
\newtheorem{definition}[theorem]{Definition}

\begin{document}
\title{MOCA: A Transformer-based Modular Causal Inference Framework with One-way Cross-attention and Cutting Feedback}
\author{
Lei Wang\\
Department of Biostatistics and Informatics\\
Colorado School of Public Health\\
\texttt{lei.2.wang@cuanschutz.edu}
\and
Debashis Ghosh\\
Department of Biostatistics and Informatics\\
Colorado School of Public Health\\
\texttt{debashis.ghosh@cuanschutz.edu}
}
\date{}
\maketitle

\begin{abstract}
{Causal effect estimation from observational data requires careful adjustment for confounding. Classical estimators such as inverse probability weighting (IPW) and augmented inverse probability weighting (AIPW) perform well under favorable model specification but may become unstable when treatment assignment and outcome mechanisms are complex, nonlinear, or high-dimensional. Machine learning and representation-learning approaches provide greater flexibility; however, joint optimization can allow outcome-related information to influence treatment-side representations, potentially compromising the intended causal structure. We propose \textbf{MOCA} (\textbf{M}odular \textbf{O}ne-way \textbf{C}ausal \textbf{A}ttention), a transformer-based framework that separates treatment and outcome modeling through a modular architecture and performs confounder adjustment via one-way cross-attention. A cutting-feedback mechanism implemented through gradient detachment further prevents outcome loss from updating the treatment module. This design preserves the flow of directional information while retaining the representational power of transformer architectures for causal inference. Using a Markov-kernel formulation, MOCA learns an autonomous treatment representation and is predictively KL-optimal within its Gaussian model class. MOCA also recovers the true average treatment effect under correct specification and that additional two way feedback cannot further reduce the error. A conformal inference procedure for individual treatment effects is proposed. Across multiple simulated scenarios, MOCA demonstrated competitive or improved performance in average treatment effect estimation compared with IPW, AIPW, X-learner, TARNet, DragonNet, BART, Causal Forest, and Do\_PFN. We further evaluated the proposed method on the Infant Health and Development Program (IHDP) dataset and the Dehejia-Wahba (DW) dataset as real-world benchmarks. An ablation study was conducted to assess the contribution of individual architectural components and support the proposed design choices. Overall, the results suggest that modular attention with one-way information flow provides an effective and interpretable framework for causal inference using modern deep learning models.}

\end{abstract}
{Keywords: Causal inference; Cross-attention; Cutting feedback; Modular learning; Treatment effect estimation; Transformer.}

\newpage
\section{Introduction}
In observational studies, treatment assignment is typically not randomized, leading to the 
presence of confounders that affect both treatment assignment and
outcomes. Many modern causal methodologies are built on the Rubin
potential outcomes framework \citep{rubin2005}. A central quantity of interest 
is the average treatment effect (ATE). 

Classical approaches can perform well when the models are correctly
specified. Propensity score methods like inverse probability weighting
(IPW) \citep{robins2000} and doubly robust estimators such as augmented inverse
probability weighting (AIPW) \citep{bang2005} are widely used approaches for
confounding adjustment. However, parametric causal estimators often rely
on restrictive assumptions, and cannot capture complex patterns such as nonlinear effects, 
or high-dimensional covariates. To address these limitations and give conditional average treatment effect (CATE) estimates, machine learning-based approaches, such as the T-learner and
X-learner have been introduced \citep{kunzel2019}. 

Tree-based nonparametric methods, including Bayesian Additive Regression Trees (BART) \citep{chipman2010bart} and Causal Forests \citep{wager2018causalforest}, further improve flexibility by automatically modeling nonlinearities and interactions. However, these methods typically perform confounding adjustment implicitly through recursive partitioning or ensemble averaging, making the adjustment mechanism less transparent and more difficult to interpret. Besides, representation-learning approaches such as TARNet \citep{johansson2016} and DragonNet
\citep{shi2019} have embedded the problem of causal inference within a 
deep learning framework consisting of representation learning paired
with two potential outcome heads. Yet joint representation learning may allow for outcome information to
affect treatment representations, which is not desirable in causal estimation. 

More recently, tabular causal inference approaches such as Do\_PFN have used pretraining to perform causal effect estimation across diverse tasks \citep{robertson2026dopfn}. Their performance depends heavily on the quality and coverage of the pretraining distribution, and their behavior may be less predictable when applied to settings that differ substantially from the pretraining data.  

To address these challenges, we draw
inspiration from recent advances in transformer architectures and
cross-attention mechanisms ~\citep{vaswani2017} combined with a novel modular inferential framework. Transformers have
demonstrated remarkable ability to capture complex dependence structures
through attention-based representation learning, while cross-attention
provides a natural mechanism for directional information transfer
between modules. Building on these ideas, we propose a modular
causal modeling framework in this article.  

In our framework, the treatment branch is designed to learn
propensity-related representations, while the outcome branch receives
treatment-side information only through one-way cross-attention. This
two-step design reflects the causal role of treatment modeling as an
adjustment mechanism for confounding while respecting the fact that outcome information should not inform the treatment assignment process \citep{ImbensRubin2015}. We also incorporate a cutting-feedback 
strategy through gradient detachment \citep{VowelsEtAl2021}, preventing the outcome loss from 
propagating backward to affect the treatment module.

Our proposed method, \textbf{MOCA} (\textbf{Modular One-way Causal
Attention}), preserves treatment modeling independence while still allowing the outcome
module to adapt to complex response surfaces. In this way, the framework
retains modern machine learning advantages while maintaining the
structural requirements of causal inference and supporting interpretable
treatment effect estimation.

\section{Problem Setup and Causal Estimand}

We consider an observational dataset consisting of \(n\) independent
units: \(\left\{ (X_{i},T_{i},Y_{i})\}_{i = 1}^{n} \right.,
\) where \(X_{i} \in \mathbb{R}^{p}\) denotes the observed covariates,
\(T_{i} \in \{ 0,1\}\) is the binary treatment indicator, and \(Y_{i}\) is
the observed outcome. Under the potential outcomes framework, each unit
\(i\) has two potential outcomes, \(Y_{i}(0)\) and \(Y_{i}(1)\),
corresponding to the outcomes that would be observed under control and
treatment~\citep{rubin2005}. The following assumptions are made. 
\begin{enumerate}
	\item Consistency: the observed
outcome equals the corresponding potential outcome under the treatment
received: \(Y_{i} = Y_{i}\left( T_{i} \right)\), implying no other version of treatment;
\item Ignorability (unconfoundedness):
\((Y_{i}(0),Y_{i}(1))\bot T_{i} \mid X_{i}.\)  That is, conditional on
the observed covariates, treatment assignment is independent of the
potential outcomes, which implies no unobserved confounding;
\item 
Positivity:
\(0 < P\left( T_{i} = 1\mid X_{i} = x \right) < 1\text{ for all }x\text{ in the support of }X,\)
which ensures that each covariate profile has a positive probability of
receiving either treatment level.
\end{enumerate}

For a unit $i$ with covariate vector $X_i = x_i$, a natural goal is to evaluate its individual level treatment effect (ITE), \(\tau_i := Y_i(1) - Y_i(0).\) However, $\tau_i$ is fundamentally unobservable, since for each unit we only observe one potential outcome. 

We instead focus on the conditional average treatment effect (CATE), \(\tau(x) := \mathbb{E}\{Y(1)-Y(0)\mid X=x\}.\) Although $\tau(x)$ is a population-level quantity, it serves as the optimal target for individual-level prediction. In particular, let $\hat{\tau}(X_i)$ denote an estimator of $\tau_i$. The mean squared error (MSE) at $X_i = x_i$ can be decomposed as
\[
\mathbb{E}\Big[(\tau_i - \hat{\tau}(X_i))^2 \mid X_i = x_i \Big]
=
\mathbb{E}\Big[(\tau_i - \tau(x_i))^2 \mid X_i = x_i \Big]
+
\mathbb{E}\Big[(\tau(x_i) - \hat{\tau}(X_i))^2 \mid X_i = x_i \Big].
\]
The first term represents irreducible uncertainty, which cannot be reduced by any estimator. Therefore, minimizing the MSE for estimating $\tau_i$ is equivalent to minimizing the MSE for estimating $\tau(x_i)$. This justifies that the best estimator for CATE is also the best estimator for ITE. 

To quantify uncertainty in individual treatment effect estimates, we apply the conformal counterfactual inference framework \citep{lei2021conformal}. We assess uncertainty quantification using empirical coverage and average interval length for ITE in the simulation study [S5]. In brief, For each treatment arm $t\in\{0,1\}$, fit the outcome model
$\hat{\mu}_t(X)$ and compute calibration residuals \(R_{j,t}=\left|Y_j-\hat{\mu}_t(X_j)\right|.\) Let $q_{t,1-\alpha}$ denote the corresponding conformal quantile. The prediction interval for the missing potential outcome is
\[
C_{i,t}
=
\left[
\hat{\mu}_t(X_i)-q_{t,1-\alpha},
\;
\hat{\mu}_t(X_i)+q_{t,1-\alpha}
\right].
\]
Combining this interval with the observed outcome gives an individual treatment effect interval for\(\tau_i=Y_i(1)-Y_i(0),\) with empirical coverage evaluated against the true ITE in simulations.

For the average treatment effect (ATE),
\(\tau = \mathbb{E}\left\lbrack Y(1) - Y(0) \right\rbrack.\) In
practice, estimation requires learning the potential outcome
\(\mu_{0}(x) = \mathbb{E}\lbrack Y(0) \mid X = x\rbrack,\mu_{1}(x) = \mathbb{E}\lbrack Y(1) \mid X = x\rbrack,\)
so that \(\tau(x) = \mu_{1}(x) - \mu_{0}(x),\ \)and
\(\tau = \mathbb{E}(\tau(x))\).

\section{Proposed Method}

We propose a modular transformer-based architecture for causal effect
estimation that explicitly separates the treatment assignment mechanism
from the potential outcome modeling mechanism [Figure 1, Figure 2].  We now describe the steps of the algorithm. 

\subsection{Scalar Feature Tokenization}

First, each scalar covariate \(X_{ij}\) is embedded independently into a
\(d\)-dimensional token. For the \(j\)-th covariate of subject \(i\),
the initial tokenizer produces \(h_{ij}^{(0)} = W_{j}X_{ij} + b_{j},h_{ij}^{(0)} \in \mathbb{R}^{d},\) where 
\(W_{j} \in \mathbb{R}^{d \times 1}\)and
\(b_{j} \in \mathbb{R}^{d}\). Because tabular covariates do not have a natural sequential order, MOCA uses feature-index embeddings ~\citep{lee2019} to distinguish covariate tokens. Thus, a learnable feature-identity
\(p_{j} \in \mathbb{R}^{d}\) is then added: \({\widetilde{h}}_{ij} = h_{ij}^{(0)} + p_{j}.\) Thus, 
for subject \(i\), the token matrix is
\(
H_i = 
\left[\widetilde{h}_{i1}^{\top},\ \widetilde{h}_{i2}^{\top},\ \ldots,\ \widetilde{h}_{ip}^{\top}\right]^{\top}
\in \mathbb{R}^{p \times d}.
\)

\subsection{Treatment Module}

\subsubsection{Linear and Self-Attention Branch}

The treatment module learns a treatment representation that summarizes
confounding information. There are two branches:

\begin{enumerate}
	\item The linear treatment branch is
\(H_{i}^{\left( L,T \right)} = {Tok}^{\left( L,T \right)}(X_{i}) \in \mathbb{R}^{p \times d}.\)
    \item The self-attention treatment branch first tokenizes \(X_{i}\):
\(H_{i}^{(S,T,0)} = {Tok}^{\left( S,T \right)}(X_{i}) \in \mathbb{R}^{p \times d}\),
then applies a Transformer encoder:
\(H_{i}^{\left( S,T \right)} = {Enc}_{T}\left( {Tok}^{\left( S,T \right)}(X_{i}) \right) \in \mathbb{R}^{p \times d}.\). 
\end{enumerate}

The encoder uses self-attention and a feed-forward subnetwork with GELU
activation \(
\operatorname{GELU}(x)
= 0.5x\left(
1 + \operatorname{erf}\left(x/\sqrt{2}\right)
\right).
\). The self-attention branch is designed to capture nonlinear feature
interactions by updating each feature representation using
input-dependent weights determined by all other feature tokens. Combined
with nonlinear feed-forward transformations, this allows the branch to
represent complex, non-additive relationships among covariates. Although
self-attention can provide some linear information, it is not the most
efficient and stable way to learn them~\citep{babiloni2023}.   The linear branch acts
as a low-complexity shortcut, which provides reliable linear
effects and reduces the computational burden on the attention branch.

\subsubsection{Treatment Gating Mechanism}

There is a cross-attention process needed to learn a 
treatment query : \(q^{(T)} \in \mathbb{R}^{d}\) which is used to summarize each branch mentioned above via multi-head
attention. To be specific, 
\[{z_{i}^{\left( L,T \right)} = MHA(q^{(T)},H_{i}^{\left( L,T \right)},H_{i}^{\left( L,T \right)}) \in \mathbb{R}^{d},
}{z_{i}^{\left( S,T \right)} = MHA(q^{(T)},H_{i}^{\left( S,T \right)},H_{i}^{\left( S,T \right)}) \in \mathbb{R}^{d}.}\]

In order to combine the information from two branches, a fusion gate is introduced with
gate weights. To be more specific, we concatenate the  two summary vectors: \(r_{i}^{(T)} = \lbrack z_{i}^{\left( L,T \right)};z_{i}^{\left( S,T \right)}\rbrack \in \mathbb{R}^{2d}.\)  A two-layer gating network is applied: \(g_{i}^{(T)} = W_{g2}^{(T)}\text{ }GELU\left( W_{g1}^{(T)}r_{i}^{(T)}+b_{g1}^{(T)} \right) + b_{g2}^{(T)} \in \mathbb{R}^{2}.\)  Next, a temperature-scaled softmax gives the treatment gate weights:
\[\alpha_{i}^{(T)} = softmax\left( \frac{g_{i}^{(T)}}{\tau_{g}} \right) = \left\lbrack \begin{array}{r}
\alpha_{iL}^{(T)} \\
\alpha_{iS}^{(T)}
\end{array} \right\rbrack,\alpha_{iL}^{(T)} + \alpha_{iS}^{(T)} = 1.
\]

In addition, a propensity score head is introduced to transform the gated
treatment representation into a treatment assignment probability. Suppose the weighted concatenation is: \(f_{i}^{(T)} = \lbrack\alpha_{iL}^{(T)}z_{i}^{\left( L,T \right)};\alpha_{iS}^{(T)}z_{i}^{\left( S,T \right)}\rbrack \in \mathbb{R}^{2d}.\) The propensity-score logit is then computed by
\({l}_{i}^{(T)} = W_{p2}\text{ }GELU(W_{p1}f_{i}^{(T)} + b_{p1}) + b_{p2},
\) and the estimated propensity score is
\({\widehat{e}}_{i} = \sigma\left( {l}_{i}^{(T)} \right),\ \)where
\(\sigma(u) = 1/(1 + e^{- u})\). At the same time, the treatment latent representation passed to the
outcome module is \(z_{i}^{(T)} \equiv \alpha_{iL}^{(T)}z_{i}^{\left( L,T \right)} + \alpha_{iS}^{(T)}z_{i}^{\left( S,T \right)} \in \mathbb{R}^{d}.\) This is converted into a single treatment token \(H_{i}^{(T)} = (z_{i}^{(T)})^{\top} \in \mathbb{R}^{1 \times d}.
\)

\subsection{Outcome Module}

The outcome module has two separate child modules; one is for the control
group \(Y_{i}(0)\) and another is for the treatment group \(Y_{i}(1)\). It
uses two separate learnable queries
\(q^{(0)} \in \mathbb{R}^{d},q^{(1)} \in \mathbb{R}^{d},\) for
estimating \(\mu_{0}(x)\) and \(\mu_{1}(x)\) from the two child modules. Within each child module, it contains three branches: 1. a
linear branch, 2. a nonlinear self-attention branch and 3. a frozen
treatment branch. For linear and self-attention branch, same as the procedure in treatment
module, we define \({H_{i}^{\left( L,Y \right)} = {Tok}^{\left( L,Y \right)}(X_{i}),
}{H_{i}^{\left( S,Y \right)} = {Enc}_{Y}\left( {Tok}^{\left( S,Y \right)}(X_{i}) \right).
}\)

For the frozen treatment branch,  a key feature of our
framework is that information flows from the treatment module to the
outcome module only, but not in the reverse direction. We refer to this
as one-way cross-attention with cutting feedback.

In many real world problems, modeling mistakes are common. From a causal-inference perspective, the treatment assignment mechanism and the outcome process play distinct roles. Accordingly, propensity score design is commonly conducted without examining outcome data \citep{rosenbaum1983,rubin2007}. If gradients from the outcome loss were allowed to update the treatment module, the learned treatment representation could be optimized for predicting the outcome. More importantly, misspecification of the outcome model could then propagate backward and distort the estimated propensity score or treatment representation. Some studies have shown that such model feedback can lead to biased propensity score estimates and poor causal effect estimation, motivating an explicit cut between the outcome and treatment components \citep{mccandless2010,zigler2013}. 

Therefore, we estimate the treatment module using the treatment objective and pass \(H_i^{(T)}\) and \(\widehat e_i\) to the outcome module through \(\operatorname{stopgrad}(\cdot)\). This preserves one-way information flow: treatment assignment information can inform outcome estimation, while the outcome model cannot alter the estimated treatment assignment mechanism. This cutting feedback idea is widely used in Modular Bayesian Inference \citep{bayarri2009}.
In our practice, the treatment module is first trained, and the frozen
treatment branch will not get updated during the outcome learning. Formally, the outcome model can be written as
\[({\widehat{\mu}}_{0}(X_{i}),{\widehat{\mu}}_{1}(X_{i})) = \mathcal{F}_{Y}(X_{i},\text{ }stopgrad(H_{i}^{(T)}),\text{ }stopgrad({\widehat{e}}_{i});\theta_{Y}).
\]

Here, \(stopgrad( \cdot )\) means that these quantities are used in
the forward pass, but gradients are not propagated back through them.
Therefore, \(\partial \mathcal{L}_Y / \partial \theta_T = 0\).
In summary, for each arm \(t \in \{ 0,1\}\), the query \(q^{(t)}\)
attends to three sources of information: 1. the outcome linear branch,
2. the outcome self-attention branch, 3. the frozen treatment branch.
Thus, similarly to the treatment module, we have 
\(z_i^{(L,t)}, z_i^{(S,t)}, z_i^{(T,t)}\), respectively.

Note that the \(z_{i}^{\left( T,t \right)}\) represents one-way cross-attention, with
the query \(q^{(t)}\) coming from the outcome module, while the key and value
come from the treatment module. Therefore, the outcome is allowed to ask
the treatment representation for information, but the treatment branch
itself does not query the outcome branch.  We then concatenate all three elements: \(r_{i}^{(t)} = \lbrack z_{i}^{\left( L,t \right)};z_{i}^{\left( S,t \right)};z_{i}^{\left( T,t \right)}\rbrack \in \mathbb{R}^{3d}.\) A fusion gate is also introduced to combine all information. For
each arm \(t\), an arm-specific gating network is used:
\[
\begin{aligned}
g_i^{(t)}
&= W_{g2}^{(t)} \, \mathrm{GELU}\!\left( W_{g1}^{(t)} r_i^{(t)} + b_{g1}^{(t)} \right) + b_{g2}^{(t)}
\in \mathbb{R}^{3}, \\
\alpha_i^{(t)}
&= \mathrm{softmax}\!\left( \frac{g_i^{(t)}}{\tau_g} \right)
= \begin{bmatrix}
\alpha_{iL}^{(t)} \\
\alpha_{iS}^{(t)} \\
\alpha_{iT}^{(t)}
\end{bmatrix},
\qquad
\alpha_{iL}^{(t)} + \alpha_{iS}^{(t)} + \alpha_{iT}^{(t)} = 1.
\end{aligned}
\]

The weighted fusion vector is
\(f_{i}^{(t)} = \lbrack\alpha_{iL}^{(t)}z_{i}^{\left( L,t \right)};\alpha_{iS}^{(t)}z_{i}^{\left( S,t \right)};\alpha_{iT}^{(t)}z_{i}^{\left( T,t \right)}\rbrack \in \mathbb{R}^{3d}.\), based on which, each arm has its own prediction head:
\({\widehat{\mu}}_{t}(X_{i}) = W_{h2}^{(t)}\text{ }GELU(W_{h1}^{(t)}f_{i}^{(t)} + b_{h1}^{(t)}) + b_{h2}^{(t)},t \in \{ 0,1\}.\)

Therefore,
\({\widehat{\tau}}_{i}\left( X_{i} \right) = {\widehat{\mu}}_{1}\left( X_{i} \right) - {\widehat{\mu}}_{0}\left( X_{i} \right),\)
which reflect the estimated CATE. And the ATE estimator is
\({\widehat{ATE}}_{\text{MeanTau}} = n^{-1}\sum_{i = 1}^{n}{\widehat{\tau}}_{i}\left( X_{i} \right)\).

\subsection{Training Objective}

The treatment and outcome modules are trained separately. First, the
treatment module is trained using the binary cross-entropy loss function: \(\mathcal{L}_{T} = - \frac{1}{n}\sum_{i = 1}^{n}\lbrack T_{i}\log{\widehat{e}}_{i} + (1 - T_{i})\log(1 - {\widehat{e}}_{i})\rbrack.\) The outcome module is trained only on observed factual outcomes. Let \(\mathcal{I}_{0} = \{ i:T_{i} = 0\},\mathcal{I}_{1} = \{ i:T_{i} = 1\}.\) Then the loss is
\[\mathcal{L}_{Y} = \frac{1}{\mid \mathcal{I}_{0} \mid}\sum_{i \in \mathcal{I}_{0}}^{}(Y_{i} - {\widehat{\mu}}_{0}(X_{i}))^{2} + \frac{1}{\mid \mathcal{I}_{1} \mid}\sum_{i \in \mathcal{I}_{1}}^{}(Y_{i} - {\widehat{\mu}}_{1}(X_{i}))^{2}.\]

\subsection{Theoretical Properties of MOCA}
\label{sec:moca-theory}

Cutting feedback originates from modular Bayesian analysis, where unwanted reverse feedback is removed during posterior updating~\citep{frazier2025, liu2025,bayarri2009}. MOCA, however, is not defined
through a joint posterior distribution. Instead, it imposes modularity through
one-way architectural dependence, staged optimization, and gradient
detachment. We therefore formulate the cut feedback principle at the level of
learned representations using Markov kernels
\citep{Kallenberg1997,fritz2020synthetic}. Proofs and technical regularity
conditions for the results in this subsection are provided in the
Supplementary Material [S4].

Let \(
D_T=\{(X_i,T_i)\}_{i=1}^n
\)
denote the treatment data and let
\(
D_O=\{Y_i\}_{i=1}^n
\)
denote the additional outcome information. Let \(Z_T\in\mathcal Z_T\)
and \(Z_Y\in\mathcal Z_Y\) denote the learned treatment and outcome
representations, respectively. A general joint learning procedure may be
represented by the Markov kernel
\(
K\!\left(dz_T,dz_Y\mid D_T,D_O\right).
\)
MOCA defines the directed cut kernel
\[
K_{\mathrm{MOCA}}\!\left(dz_T,dz_Y\mid D_T,D_O\right)
=
K_T\!\left(dz_T\mid D_T\right)
K_Y\!\left(dz_Y\mid D_T,D_O,z_T\right).
\]

\begin{proposition}[Representation-level cutting feedback]
\label{prop:moca-treatment-autonomy}
Under standard measurability conditions, the directed product above defines
a valid Markov kernel. Moreover, its treatment representation marginal
satisfies
\(
K_{\mathrm{MOCA},T}\!\left(dz_T\mid D_T,D_O\right)
=
K_T\!\left(dz_T\mid D_T\right).
\)

Consequently, the distribution of the learned treatment representation is
invariant to the outcome information \(D_O\). The outcome module may
condition on \(Z_T\), but outcome information cannot feed back and alter
the treatment representation.
\end{proposition}

MOCA realizes this factorization algorithmically. Let \(U_T\) and \(U_Y\)
denote randomness arising from initialization, mini-batch construction, or
optimization. The staged learning procedure can be written as
\(
\widehat{\theta}_T
=
\mathcal A_T(D_T,U_T),
\widehat{\theta}_Y
=
\mathcal A_Y(D_T,D_O,\widehat{\theta}_T,U_Y),
\)
where \(\widehat{\theta}_T\) is held fixed during outcome optimization. In
particular,
\(\partial \mathcal{L}_Y / \partial \theta_T = 0\).
Because the treatment-learning algorithm \(\mathcal A_T\) has no dependence
on \(D_O\), staged optimization with gradient detachment enforces the
treatment-autonomy property in Proposition~\ref{prop:moca-treatment-autonomy}.

After the treatment module has been learned and fixed, write
\(
Z_T=h_T(X),W=(X,Z_T).
\)
For treatment arm \(t\in\{0,1\}\), let
\(
P_t(dy\mid w)
=
\mathcal L(Y\mid W=w,T=t)
\)
denote the true conditional outcome law. For a prediction function \(f_t\),
consider the fixed-variance Gaussian predictive kernel
\[
Q_{f,t}(dy\mid w)
=
\mathcal N\!\left(dy;f_t(w),\sigma_t^2\right),
\qquad
\sigma_t^2>0.
\]

\begin{theorem}[Gaussian predictive KL-MSE equivalence]
\label{thm:kl-mse}
Suppose that \(P_t(\cdot\mid W)\) has a finite second moment and that the
expected conditional forward KL divergence is finite. Then
\[
\begin{aligned}
\mathcal K_t(f_t)
&=
\mathbb E_{W\mid T=t}
\left[
\operatorname{KL}
\left\{
P_t(\cdot\mid W)
\,\Vert\,
Q_{f,t}(\cdot\mid W)
\right\}
\right] \\
&=
C_t
+
\frac{1}{2\sigma_t^2}
\mathbb E
\left[
\{Y-f_t(W)\}^2
\mid T=t
\right],
\end{aligned}
\]
where \(C_t\) does not depend on \(f_t\). Consequently, for any prediction
class \(\mathcal F_t\),
\[
\arg\min_{f_t\in\mathcal F_t}\mathcal K_t(f_t)
=
\arg\min_{f_t\in\mathcal F_t}
\mathbb E
\left[
\{Y-f_t(W)\}^2
\mid T=t
\right].
\]
Thus, the MSE-trained MOCA outcome head is predictively KL-optimal within
the corresponding fixed-variance Gaussian predictive class.
\end{theorem}

The population masked outcome risk is
\(
\mathcal R_Y(f_0,f_1)
=
\mathbb E
\left[
T\{Y-f_1(W)\}^2
+
(1-T)\{Y-f_0(W)\}^2
\right].
\)
Because the two outcome heads enter separate terms, their population targets
can be characterized arm by arm.

\begin{theorem}[Population target and ATE recovery]
\label{thm:moca-ate}
For \(t\in\{0,1\}\), define
\[
R_t(f_t)
=
\mathbb E
\left[
\{Y-f_t(W)\}^2
\mid T=t
\right],
\qquad
m_t(W)
=
\mathbb E(Y\mid W,T=t).
\]
For every square integrable \(f_t\),
\(
R_t(f_t)
=
R_t(m_t)
+
\left\|f_t-m_t\right\|_{L^2(P_{W\mid T=t})}^2.
\)
Therefore, within a restricted MOCA function class \(\mathcal F_t\), the
population minimizer is the best \(L^2(P_{W\mid T=t})\) approximation to the
treatment specific conditional mean.

Suppose additionally that consistency, conditional ignorability, and positivity hold. Because \(Z_T=h_T(X)\) and the MOCA outcome module
retains \(X\),
\[
m_t\{X,h_T(X)\}
=
\mathbb E(Y\mid X,T=t)
=
\mathbb E\{Y(t)\mid X\}
=
\mu_t(X).
\]
If the MOCA function class contains the maps
\(
\{x,h_T(x)\}\longmapsto\mu_t(x), t\in\{0,1\},
\)
and the population minimizers are attained, then
\(
f_t^*\{X,h_T(X)\}=\mu_t(X), t\in\{0,1\}.
\)
Consequently, the population plug-in estimand satisfies
\[
\begin{aligned}
\psi_{\mathrm{MOCA}}^*
&=
\mathbb E
\left[
f_1^*\{X,h_T(X)\}
-
f_0^*\{X,h_T(X)\}
\right] \\
&=
\mathbb E\{\mu_1(X)-\mu_0(X)\}
=
\mathbb E\{Y(1)-Y(0)\}
=
\psi.
\end{aligned}
\]
\end{theorem}

The preceding result also gives a direct characterization under
misspecification. For arbitrary outcome predictors \(f_0\) and \(f_1\), let
\(
e_t(X)
=
f_t\{X,h_T(X)\}-\mu_t(X).
\)
Then the population plug-in ATE error satisfies
\(
\left|\psi_f-\psi\right|
\leq
\left\|e_1\right\|_{L^2(P_X)}
+
\left\|e_0\right\|_{L^2(P_X)}.
\)
Moreover, let
\(
p_t(X)=P(T=t\mid X), \pi_t=P(T=t),
\)
and suppose that strong overlap holds:
\(
p_t(X)\geq\epsilon_t>0
\quad\text{almost surely}.
\)
Defining the arm-specific excess MSE risk by
\(
\mathcal E_t(f_t)
=
R_t(f_t)-R_t(m_t),
\)
we obtain
\[
\left|\psi_f-\psi\right|
\leq
\sum_{t=0}^1
\sqrt{\frac{\pi_t}{\epsilon_t}}
\sqrt{\mathcal E_t(f_t)}.
\]
Hence, small arm-specific excess prediction risk implies small population
ATE error, with the guarantee becoming weaker as overlap deteriorates.

Finally, the effect of allowing reverse feedback can be characterized
directly. Let \(f_t^{\mathrm{cut}}\) denote the outcome functions obtained
under one-way cutting, and suppose that outcome to treatment feedback changes
them to
\(
f_t^{\mathrm{fb}}(X)
=
f_t^{\mathrm{cut}}(X)+\Delta_t(X).
\)
Define
\(
\delta_{\mathrm{fb}}
=
\mathbb E\{\Delta_1(X)-\Delta_0(X)\}
\)
and let
\(
b_{\mathrm{cut}}
=
\psi_{\mathrm{cut}}-\psi
\)
denote the population ATE bias of the cut solution.

\begin{proposition}[Feedback-induced ATE distortion]
\label{prop:feedback-distortion}
Let
\(
\mathcal R_{\mathrm{cut}}
=
(\psi_{\mathrm{cut}}-\psi)^2,
\mathcal R_{\mathrm{fb}}
=
(\psi_{\mathrm{fb}}-\psi)^2.
\)
Then
\(
\mathcal R_{\mathrm{fb}}
-
\mathcal R_{\mathrm{cut}}
=
2b_{\mathrm{cut}}\delta_{\mathrm{fb}}
+
\delta_{\mathrm{fb}}^2.
\)
If the cut solution is ATE-correct,
\(
\psi_{\mathrm{cut}}=\psi,
\), 
\(
\mathcal R_{\mathrm{fb}}
-
\mathcal R_{\mathrm{cut}}
=
\delta_{\mathrm{fb}}^2
=
\left[
\mathbb E\{\Delta_1(X)-\Delta_0(X)\}
\right]^2
\geq 0.
\)
Thus, additional outcome to treatment feedback cannot improve an already
ATE correct cut solution. It preserves the same ATE only when the average
feedback induced perturbations in the two treatment arms cancel exactly.
\end{proposition} 

The theoretical
justification for cutting feedback is that reverse feedback
offers no population level ATE improvement once the autonomous cut solution
has correctly recovered the effect.

\section{Simulation Study}

\subsection{Simulation}
To comprehensively evaluate the performance of the proposed method, we
conducted seven simulation studies under multiple data-generating
mechanisms including linear, linear with outcome model misspecification, non-linear, t-distribution, hidden-confounding, and high dimensional [Table 1, S2]. Among the simulation scenarios, the linear design serves as the correctly specified benchmark. The linear-Omis scenario evaluates estimator stability when the outcome model is misspecified while the treatment model remains correctly specified. The nonlinear and high-dimensional scenarios assess nonlinear approximation capacity and performance under high-dimensional estimation and regularization, respectively.  The heavy-tailed scenario retains linear treatment and outcome structures but evaluates robustness to extreme covariate values. The two hidden-confounding scenarios violate the ignorability assumption and therefore represent failures of causal identification caused by omitted confounders.

Including MOCA, there
are 9 methods compared in the simulation. IPW \citep{robins2000} and AIPW \citep{bang2005} are
well-established methods in classical causal inference, and AIPW is
especially attractive due to its doubly robustness property. 
The X-learner \citep{kunzel2019} represents a class of machine-learning based approaches that have shown
competitive performance in treatment effect estimation. TARNet \citep{johansson2016} and
DragonNet \citep{shi2019} are representation-learning methods and provide important
neural network based benchmarks for comparison. BART \citep{chipman2010bart} and Causal Forest \citep{wager2018causalforest} are two tree-based methods, providing a non-parametric comparison. The Do\_PFN \citep{robertson2026dopfn}, is a recent tabular causal model, as a tabular baseline [S1].

\subsection{Ablation Experiments}
To evaluate the contribution of each component in MOCA, we considered several architectural variants. Each variant removes or modifies one key design element while keeping the remaining training procedure unchanged [S3]. Specifically, we considered six variants: Remove Linear Branch, Remove Treatment Branch, Concatenation Fusion, One-way without Cutting Feedback, Two-way with Cutting Feedback, and Two-way without Cutting Feedback. These experiments directly evaluate the contributions of the linear branch, treatment representation, fusion mechanism, and cutting-feedback constraint.

All simulations and model training were conducted in Python (3.11.4). The proposed model was implemented in PyTorch (version 2.5.1+cu121),
with NumPy (version 2.4.3) and Pandas (3.0.1) used for data generation
and numerical computation. Random seeds with different seeds for different samples were fixed to ensure
reproducibility. We considered 100 observations in each of the training, validation, and test sets. The hyperparameters are set to the same during the simulation [Table S1]. 

\section{Results}

\subsection{Simulation Results}

The simulation results suggest that MOCA is competitive across a wide range of settings [Table 2]. For ATE estimation, it achieved among the smallest bias mean and RMSE mean in the linear-Omis, nonlinear, hidden-confounding, and high-dimensional scenarios. In the correctly specified linear setting, MOCA attained an ATE bias mean of 0.174 and an RMSE mean of 0.459. Under outcome-model misspecification, it remained stable, achieving the smallest ATE bias mean of 0.204 and an RMSE mean of 0.579. By contrast, the X-learner performed best in the linear setting, with a bias mean of 0.044 and an RMSE mean of 0.303, but deteriorated in the linear-Omis scenario, where its bias mean and RMSE mean increased to 1.075 and 1.133. This comparison supports the robustness of MOCA’s cutting-feedback design.

MOCA also achieved the smallest mean ATE bias and RMSE in the nonlinear and high-dimensional settings, indicating strong nonlinear approximation and regularization performance. Under (t)-distributed covariates, the X-learner performed slightly better because the treatment and outcome mechanisms remained linear and were well aligned with its regression-based structure \citep{kunzel2019}. TARNet and DragonNet were less stable under heavy-tailed data, possibly because shared representations trained with squared-error loss are sensitive to extreme outcomes \citep{fan2021}. Nevertheless, MOCA remained robust, with an ATE bias mean of 0.131 and an RMSE mean of 0.615. 

Under violations of ignorability, MOCA generally had smaller bias than competing methods, suggesting that directional separation may reduce contamination between treatment and outcome learning when the model is misspecified or information is incomplete. However, valid estimation still depends on appropriate identifying assumptions \citep{vanderweele2011}. MOCA also produced fewer extreme estimates than IPW and AIPW [Figure S1–S2]. These traditional methods can become unstable when estimated propensity scores approach zero or one, leading to extreme weights, outliers, and larger RMSE \citep{rosenbaum1983,li2019}.

For CATE estimation, MOCA performed similarly to most machine-learning methods [Table 2]. Its more limited improvement may arise because all individuals share a common frozen treatment representation, which captures population-level treatment information but may be less flexible for individual-level heterogeneity. For conformal ITE inference, MOCA provided the most stable coverage, generally close to or above the nominal $95\%$ level. Causal Forest was also relatively well calibrated but often produced wider intervals, whereas TARNet, DragonNet, BART, and X-Learner showed undercoverage in hidden-confounding scenarios. Do-PFN was less stable, overcovering with wide intervals in high-dimensional settings while undercovering in linear and heavy-tailed settings [Table S2].

In the bias–RMSE plots, MOCA generally lies near the lower-left corner, reflecting both low RMSE and bias close to zero [Figure 3]. Its advantage is most apparent in the linear-Omis, nonlinear, high-dimensional, and hidden-confounding scenarios. In simpler settings, particularly the (t)-distribution scenario, several methods perform similarly, although MOCA remains among the strongest overall.

MOCA also offers a transparent architecture by separating treatment and outcome learning into dedicated modules. It provides an explicit treatment representation and propensity score output, enabling standard overlap and covariate balance diagnostics. Unlike Do-PFN, it is trained directly on the target dataset and does not depend on large-scale pretraining that may be mismatched to the application domain. Overall, MOCA outperforms IPW and AIPW in most challenging scenarios while remaining competitive with other machine-learning methods.

\subsection{Ablation results}
Across all ablation settings, the full MOCA architecture achieved the smallest
average ATE bias, with an ATE bias mean of \(-0.029\) [Table 3, Figure S3]. The results show that
removing the cutting-feedback mechanism leads to substantially larger errors in
both the one-way and two-way architectures, highlighting the importance of
preventing outcome-loss gradients from contaminating the treatment
representation. The concatenation fusion variant also performed worse than gated fusion,
suggesting that adaptive weighting of treatment and outcome representations is
beneficial.

The linear branch had a relatively modest effect in this setting because the DGP
is primarily driven by nonlinear components, and the gated model assigns smaller
weights to the linear representation. Nevertheless, removing the linear branch
still increased ATE bias, indicating that low-order additive signals remain
useful. Removing the treatment branch also increased both bias and RMSE,
supporting the role of propensity-informed treatment representations. Finally,
the two-way variants produced larger bias, consistent with the concern that
bidirectional information flow violates the intended modular causal design.
Overall, the ablation results support the architectural choices in MOCA and, in
particular, demonstrate the importance of the cutting-feedback mechanism.

\subsection{Real Data Results}

Two open-sourced real datasets are used here [Table 4]. We estimated uncertainty for MOCA by repeatedly fitting the model and combining the resulting ATE estimates using a Rubin-style pooling rule~\citep{barnard1999}. IHDP refers to the Infant Health and Development Program dataset, which comes from a randomized
experiment conducted between 1985 and 1988 to study the effect of
home-based intervention on infants' cognitive test scores~\citep{louizos2017}.
There are 747 observations with 608 controls and 139 treatments. This
version contains the treatment indicator, the observed outcome
\(y_{\text{factual}}\), the counterfactual outcome
\(y_{\text{cfactual}}\), the mean potential outcomes under control and
treatment (\(\mu_{0}\) and \(\mu_{1}\)), and 25 covariates. Among these,
\(X_{1}\)--\(X_{6}\) are continuous variables, while
\(X_{7}\)--\(X_{25}\) are binary variables. A counterfactual outcome included in this dataset, so that we can show the ATE bias directly. The IHDP results show that MOCA has the smallest ATE bias of 0.006. We have a pooling estimated ATE of 3.991 and a 95\% confidence interval of (3.69, 4.29).

DW refers to the Dehejia--Wahba dataset, which is a subset of the
Lalonde data~\citep{dehejia1999}.  It contains 445 observations with 185 treatments
and 260 controls. Typically, we expected ATE is a positive number, and the
doubly robust method AIPW can be a good reference~\citep{iacus2019}. The
variables include the treatment indicator treated, age, education, race,
marital status, no-degree status, earnings in 1974 and 1975,
post-treatment earnings in 1978, and unemployment indicators u74 and
u75. Among these, re78 is usually treated as the post-treatment outcome.
However, this dataset does not have a truth for reference. In
this case, we will report only the ATE instead. The ATE based on doubly
robust method AIPW is 784.214 with a 95\% confidence interval of (-1746.59, 3315.02), which is a potential reference for other methods. MOCA pooling estimates ATE to be 610.776 and a 95\% confidence interval of (-163.01, 1384.57), which is closest to the AIPW method.

\section{Discussion}

This paper introduces a modular causal transformer
framework that combines representation learning, one-way
cross-attention, and cutting feedback for treatment effect estimation.
The proposed design aims to preserve the usefulness of treatment
information for confounding adjustment while avoiding undesirable
reverse influence from outcome learning to treatment representation.

At the same time, several limitations should be acknowledged. First, in relatively simple settings, the added
complexity of the model may not translate into substantial improvement
over simpler approaches. Richer architectures can better capture complex
structure, but may also overfit or underperform when the true mechanism
is simple or when the available sample size is limited \citep{hastie2009}.  Second,
the validity of the MOCA still relies on standard identification
assumptions such as consistency, ignorability, and positivity. The model
cannot fully correct for hidden confounding when important variables are
entirely unobserved. Third, all individuals share a common frozen treatment representation. While this design stabilizes treatment-effect estimation at the population level, it may introduce additional noise in individual-level effect estimation. Motivated by this observation, one of our future directions is to combine MOCA with individual level treatment representation  to estimate uncertainty-aware ITEs rather than relying solely on point predictions. In
addition, future work could also explore methods that combine the proposed
architecture with sensitivity analysis or proxy-variable strategies to
better address hidden confounding.

\newpage
\section{Tables and Figures}

\begin{table}[!ht]
\centering
\caption{Summary of the simulation scenarios.}
\label{tab:simulation_scenarios}
\small
\setlength{\tabcolsep}{4pt}
\renewcommand{\arraystretch}{1.12}

\begin{tabular}{
>{\raggedright\arraybackslash}p{2.8cm}
>{\raggedright\arraybackslash}p{3.5cm}
>{\raggedright\arraybackslash}p{3.5cm}
>{\raggedright\arraybackslash}p{4.0cm}
}
\hline
Scenario
&
Treatment mechanism
&
Outcome mechanism
&
Primary purpose
\\
\hline

Linear
&
Linear
&
Linear
&
Correctly specified baseline
\\

Linear-Omis
&
Linear
&
Linear with omitted \(U\)
&
Outcome model misspecification
\\

Nonlinear
&
Nonlinear
&
Nonlinear
&
Nonlinear structures
\\

Heavy-tailed
&
\(t\)-distributed
&
\(t\)-distributed
&
Heavy-tailed covariates
\\

Hidden confounding
&
Unobserved \(U\)
&
Unobserved \(U\)
&
Single unobserved confounder
\\

Hidden confounding, multiple-\(U\)
&
Multiple unobserved \(U\)
&
Multiple unobserved \(U\)
&
Multiple unobserved confounders
\\

High-dimensional
&
Linear
&
Linear with 300 covariates
&
High-dimensional estimation and regularization
\\

\hline
\end{tabular}
\end{table}

\begin{table}[!ht]
\centering
\caption{Summary simulation results with a test-sample size of 100.}
\label{tab:simulation-summary-transposed}

\setlength{\tabcolsep}{3pt}
\renewcommand{\arraystretch}{0.95}

\resizebox{\linewidth}{!}{%
\begin{tabular}{llccccccccc}
\toprule
Scenario & Metric
& IPW
& AIPW
& X-Learner
& MOCA
& TARNet
& DragonNet
& BART
& CausalForest
& Do\_PFN \\
\midrule

Linear
& ATE Bias
& -5.128 & -0.692 & 0.044 & 0.174 & 0.894 & 0.750 & 1.067 & 1.866 & 0.235 \\
& ATE RMSE
& 6.536 & 1.089 & 0.303 & 0.459 & 0.901 & 0.758 & 1.074 & 1.866 & 0.436 \\
& CATE RMSE
& -- & -- & 0.751 & 2.390 & 1.897 & 1.814 & 1.907 & 3.072 & 2.234 \\
\midrule

Linear-Omis
& ATE Bias
& -2.039 & -0.473 & 1.075 & 0.204 & 0.799 & 0.753 & 0.845 & 1.560 & 0.821 \\
& ATE RMSE
& 3.405 & 1.309 & 1.133 & 0.579 & 0.805 & 0.768 & 0.847 & 1.560 & 0.842 \\
& CATE RMSE
& -- & -- & 3.519 & 2.313 & 1.949 & 1.919 & 1.817 & 2.909 & 2.439 \\
\midrule

Non-linear
& ATE Bias
& -0.311 & 1.350 & 0.778 & -0.214 & 0.996 & 0.901 & 0.901 & 1.297 & 0.414 \\
& ATE RMSE
& 2.723 & 1.947 & 0.778 & 0.521 & 0.996 & 0.916 & 0.901 & 1.297 & 0.597 \\
& CATE RMSE
& -- & -- & 1.777 & 2.621 & 2.092 & 2.052 & 2.152 & 2.891 & 2.720 \\
\midrule

\(t\)-Distribution
& ATE Bias
& -23.019 & -1.171 & 0.025 & 0.131 & 1.334 & 1.180 & 1.690 & 2.573 & 0.240 \\
& ATE RMSE
& 24.215 & 5.827 & 0.244 & 0.615 & 1.341 & 1.196 & 1.690 & 2.573 & 0.533 \\
& CATE RMSE
& -- & -- & 0.803 & 3.603 & 3.020 & 2.919 & 3.410 & 4.483 & 3.509 \\
\midrule

Hidden Confounding
& ATE Bias
& 0.343 & 2.018 & 1.661 & 0.712 & 2.558 & 2.553 & 2.672 & 3.008 & 1.668 \\
& ATE RMSE
& 3.525 & 3.124 & 1.661 & 0.799 & 2.558 & 2.553 & 2.672 & 3.008 & 1.668 \\
& CATE RMSE
& -- & -- & 2.357 & 3.092 & 3.493 & 3.488 & 3.610 & 4.251 & 3.198 \\
\midrule

Hidden Conf. MultiU
& ATE Bias
& 1.381 & 0.917 & 2.445 & 1.121 & 3.274 & 3.256 & 3.384 & 3.551 & 2.225 \\
& ATE RMSE
& 3.801 & 3.766 & 2.445 & 1.138 & 3.274 & 3.256 & 3.384 & 3.551 & 2.225 \\
& CATE RMSE
& -- & -- & 3.112 & 3.284 & 4.141 & 4.107 & 4.254 & 4.687 & 3.604 \\
\midrule

High Dimension
& ATE Bias
& 31.348 & -14.692 & 4.126 & 0.063 & 3.731 & 3.721 & 4.506 & 4.187 & -1.022 \\
& ATE RMSE
& 31.822 & 22.117 & 4.126 & 0.624 & 3.731 & 3.721 & 4.506 & 4.187 & 1.031 \\
& CATE RMSE
& -- & -- & 5.710 & 3.956 & 5.362 & 5.344 & 6.116 & 5.741 & 4.039 \\

\bottomrule
\end{tabular}%
}

\end{table}

\begingroup
\small
\setlength{\tabcolsep}{3pt}
\begin{longtable}{lcccc}
\caption{Ablation study of MOCA.}\label{tab:ablation}\\
\toprule
Ablation &
ATE Bias Mean &
ATE Bias SD &
ATE RMSE &
CATE RMSE \\
\midrule
\endfirsthead

\multicolumn{5}{c}%
{{\tablename\ \thetable{} -- continued from previous page}} \\
\toprule
Ablation &
ATE Bias Mean &
ATE Bias SD &
ATE RMSE &
CATE RMSE \\
\midrule
\endhead

\midrule
\multicolumn{5}{r}{{Continued on next page}} \\
\endfoot

\bottomrule
\endlastfoot

MOCA &
-0.029 &
0.815 &
0.625 &
3.586 \\

Concatenation Fusion &
0.796 &
0.691 &
0.906 &
3.578 \\

Remove Linear Shortcut &
-0.157 &
0.749 &
0.602 &
3.571 \\

Remove Treatment Branch &
0.157 &
0.863 &
0.689 &
3.679 \\

Without Cutting Feedback &
0.994 &
0.953 &
1.115 &
3.533 \\

Two-way with Cutting Feedback &
1.001 &
0.962 &
1.159 &
3.545 \\

Two-way without Cutting Feedback &
-0.236 &
0.737 &
0.610 &
3.562 \\

\end{longtable}
\endgroup

\begin{table}[H]
\caption{Results for the IHDP and DW real-data analyses.}
\label{tab:real-data}
\centering
\resizebox{\textwidth}{!}{%
\begin{tabular}{l l c c c c c}
\toprule
\multicolumn{1}{c}{Dataset} & Method & ATE & ATE Bias & Pooling ATE & Pooling ATE Bias & 95\% CI \\
\midrule

\multirow{9}{*}{IHDP Data}
& IPW           & 8.866 & 4.911  & ---   & ---    & (0.60, 17.13) \\
& AIPW          & 3.325 & -0.630 & ---   & ---    & (1.80, 4.85) \\
& X-Learner     & 3.861 & -0.094 & 3.905 & -0.061 & (3.72, 4.08) \\
& MOCA          & 3.961 & 0.006  & 3.991 & 0.025  & (3.69, 4.29) \\
& TARNet        & 3.791 & -0.165 & 4.025 & 0.060  & (3.58, 4.47) \\
& DragonNet     & 3.837 & -0.119 & 3.927 & -0.039 & (3.36, 4.50) \\
& BART          & 3.876 & -0.079 & 3.899 & -0.067 & (3.71, 4.09) \\
& Causal Forest & 3.821 & -0.134 & 3.855 & -0.110 & (3.54, 4.17) \\
& Do\_PFN        & 2.792 & -1.163 & 2.599 & -1.367 & (2.18, 3.02) \\

\midrule

\multirow{9}{*}{DW Data}
& IPW           & 1067.148 & --- & ---      & --- & (-2277.80, 4412.10) \\
& AIPW          & 784.215  & --- & ---      & --- & (-1746.59, 3315.02) \\
& X-Learner     & 2216.231 & --- & 1754.876 & --- & (210.45, 3299.30) \\
& MOCA          & 633.005  & --- & 610.776  & --- & (-163.01, 1384.56) \\
& TARNet        & 1850.076 & --- & 1161.621 & --- & (24.02, 2299.22) \\
& DragonNet     & 1525.850 & --- & 1262.497 & --- & (209.60, 2315.39) \\
& BART          & 1902.926 & --- & 1740.527 & --- & (377.02, 3104.03) \\
& Causal Forest & 1371.187 & --- & 1555.570 & --- & (721.52, 2389.62) \\
& Do\_PFN        & 1740.946 & --- & 1211.340 & --- & (364.38, 2058.30) \\

\bottomrule
\end{tabular}%
}
\end{table}

\newpage
\begin{figure}[htbp]
    \centering
    \includegraphics[width=5.0in,height=2.8in]{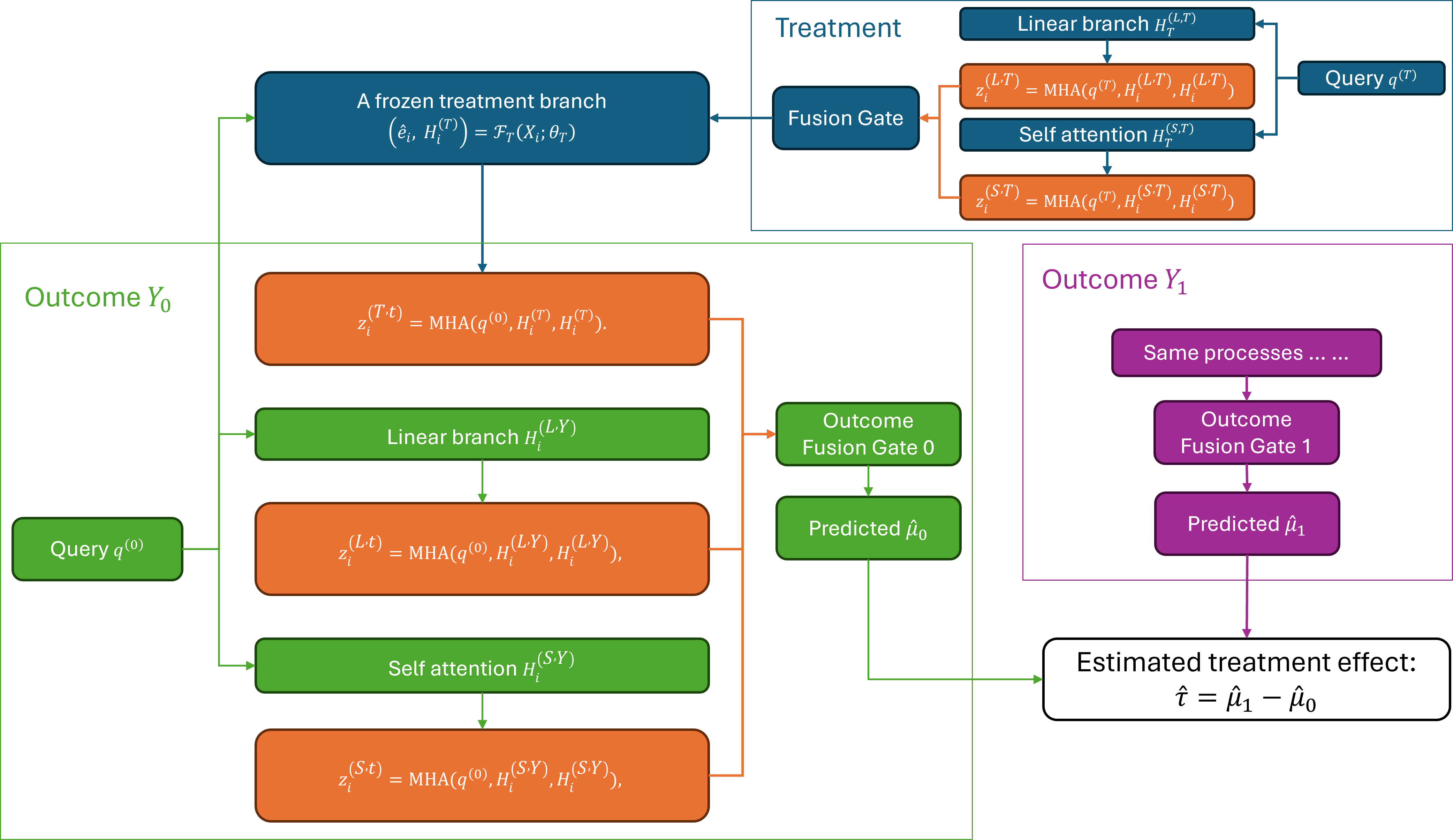}
    \caption{Illustration of the MOCA framework with one-way attention.}\label{fig:moca-framework}
\end{figure}

\begin{figure}[H]
    \centering
    \includegraphics[width=0.55\linewidth]{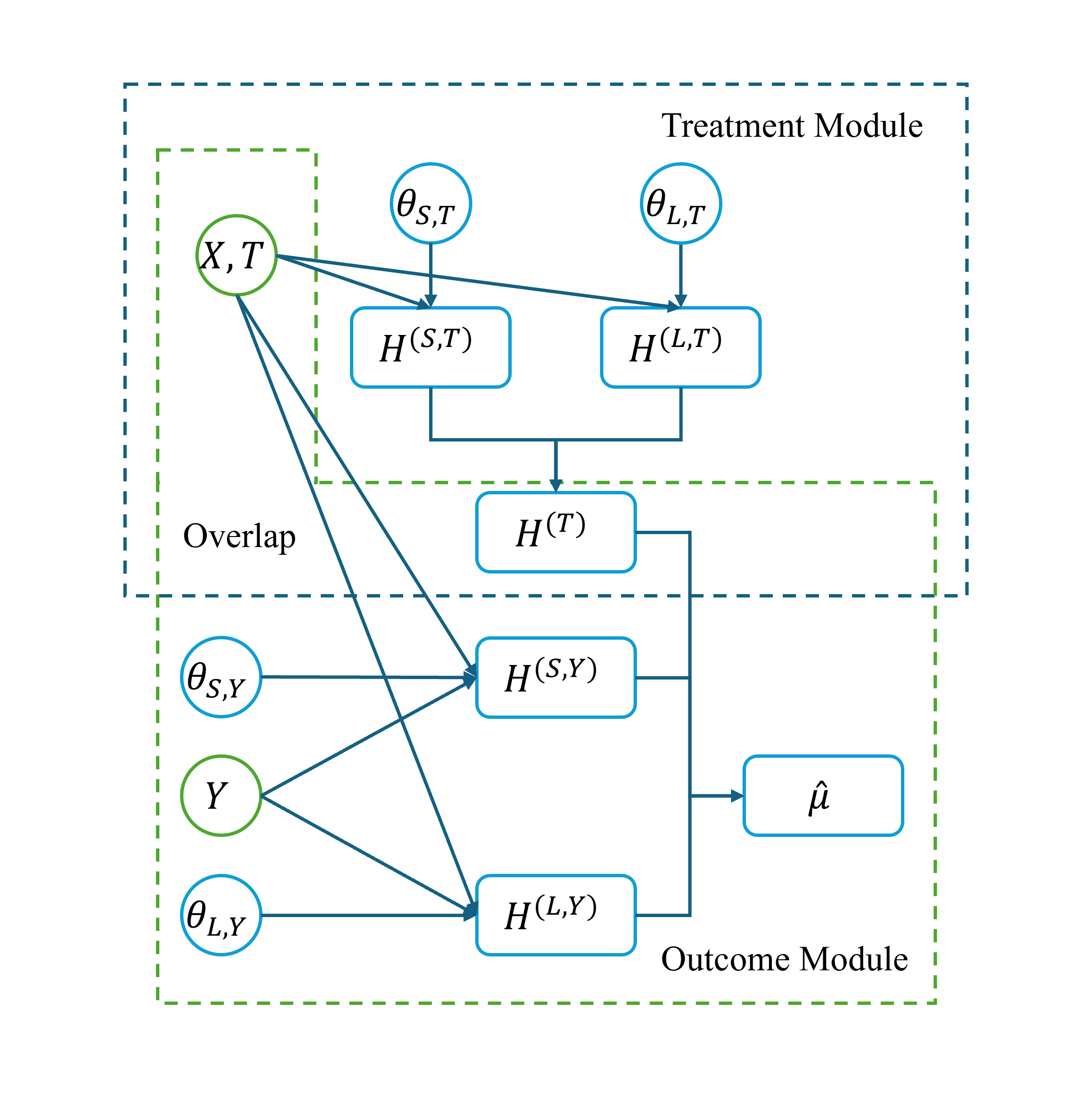}
    \caption{Information flow in MOCA with treatment and outcome modules.}\label{fig:moca-flow}
\end{figure}

\begin{figure}[htbp]
    \centering
    \includegraphics[width=1\linewidth]{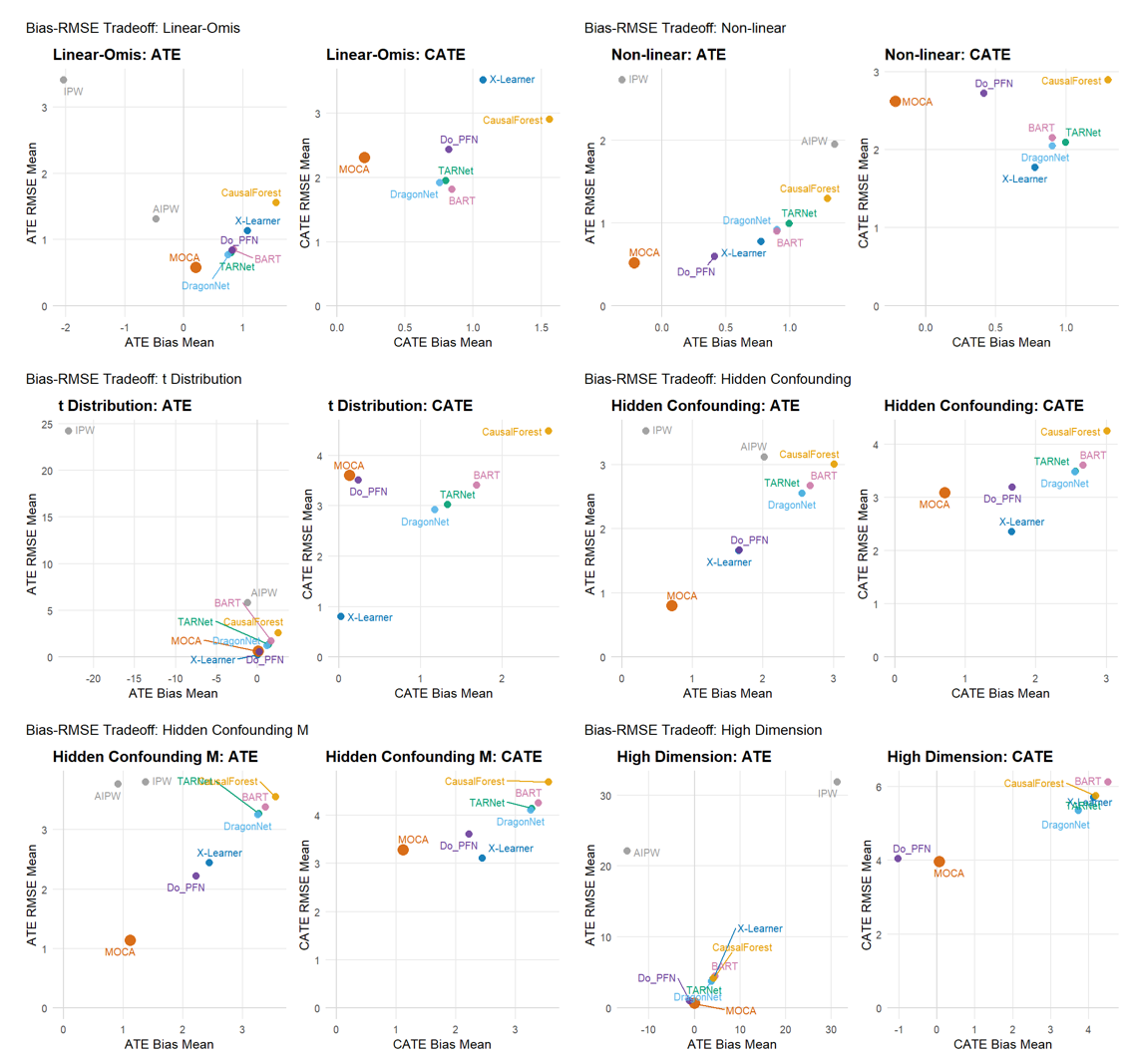}
    \caption{Bias-RMSE tradeoff for simulation results}\label{fig:moca-results}
\end{figure}

\clearpage
\newpage

\clearpage
\appendix
\begin{center}
{\LARGE\bfseries Supplementary Material for:}\\[0.75em]
{\Large\bfseries MOCA: A Transformer-based Modular Causal Inference Framework with One-way Cross-attention and Cutting Feedback}
\end{center}
\vspace{1.5em}

\setcounter{section}{0}
\setcounter{subsection}{0}
\renewcommand{\thesection}{S\arabic{section}}
\renewcommand{\thesubsection}{S\arabic{section}.\arabic{subsection}}
\renewcommand{\theHsection}{supp.S\arabic{section}}
\renewcommand{\theHsubsection}{supp.S\arabic{section}.\arabic{subsection}}
\setcounter{table}{0}
\renewcommand{\thetable}{S\arabic{table}}
\renewcommand{\theHtable}{supp.S\arabic{table}}
\setcounter{figure}{0}
\renewcommand{\thefigure}{S\arabic{figure}}
\renewcommand{\theHfigure}{supp.S\arabic{figure}}

\section{Related Work}
\subsection{Classical Causal Inference Estimators}

The propensity score is defined as the conditional probability of
treatment assignment given observed covariates,
\[e(x) = P(T = 1 \mid X = x).\]

By balancing covariates between treated and control groups, the
propensity score can be used for matching and weighting. A common
weighting estimator is inverse probability weighting (IPW), \citep{robins2000}
which estimates the ATE as
\[{\widehat{\tau}}_{IPW} = \frac{1}{n}\sum_{i = 1}^{n}\left( \frac{T_{i}Y_{i}}{\widehat{e}(X_{i})} - \frac{(1 - T_{i})Y_{i}}{1 - \widehat{e}(X_{i})} \right).
\]

A doubly robust estimator combines propensity score estimation and
outcome regression. One example is the augmented inverse probability
weighting (AIPW) estimator \citep{bang2005}:
\[{\widehat{\tau}}_{AIPW} = \frac{1}{n}\sum_{i = 1}^{n}\left\lbrack {\widehat{\mu}}_{1}(X_{i}) - {\widehat{\mu}}_{0}(X_{i}) + \frac{T_{i}\{ Y_{i} - {\widehat{\mu}}_{1}(X_{i})\}}{\widehat{e}(X_{i})} - \frac{(1 - T_{i})\{ Y_{i} - {\widehat{\mu}}_{0}(X_{i})\}}{1 - \widehat{e}(X_{i})} \right\rbrack.\]

This estimator is consistent if either the treatment model or the
outcome regression model is correctly specified, which makes it doubly robust.

\subsection{Machine Learning Based Causal Effect Estimation}

Meta-learning approaches adapt supervised learning algorithms to estimate treatment effects. The T-learner fits two separate outcome models, using the control and treated samples respectively, and estimates the treatment effect by
\(\widehat{\tau}(x) = {\widehat{\mu}}_{1}(x) - {\widehat{\mu}}_{0}(x).\)

This approach is simple and flexible, but may suffer when one treatment
group has limited sample size. The X-learner improves upon the T-learner
by first imputing pseudo-treatment effects and then learning treatment
effect functions separately in each group, often achieving better
performance under treatment imbalance \citep{kunzel2019}. In our analysis, we will compare to X-learner since it considers the treatment adjustments. 

Representation learning methods for causal inference seek to learn
latent covariate representations that are useful for estimating both
treatment assignment and potential outcomes. TARNet and DragonNet are
two examples.

Let \(\Phi(X)\) denote a learned representation of the covariates.
TARNet predicts the two potential outcome surfaces through separate
heads \citep{johansson2016}:
\[{\widehat{\mu}}_{0}(X) = h_{0}(\Phi(X)),{\widehat{\mu}}_{1}(X) = h_{1}(\Phi(X)).\]

The factual prediction is
\[\widehat{Y} = T{\widehat{\mu}}_{1}(X) + (1 - T){\widehat{\mu}}_{0}(X),\]

and the training objective is typically based on factual outcome loss,
\[\mathcal{L}_{TARNet} = \frac{1}{n}\sum_{i = 1}^{n}\ell\!\left( Y_{i},\text{ }T_{i}{\widehat{\mu}}_{1}(X_{i}) + (1 - T_{i}){\widehat{\mu}}_{0}(X_{i}) \right),
\]

where \(\ell(\cdot,\cdot)\) is usually mean squared error for continuous outcomes or cross-entropy for binary outcomes.

DragonNet extends TARNet by adding an explicit treatment prediction
head \citep{shi2019}. In addition to the shared representation \(\Phi(X)\) and
the outcome heads \(h_{0}(\Phi(X))\) and \(h_{1}(\Phi(X))\), DragonNet
estimates the propensity score through
\[\widehat{e}(X) = g(\Phi(X)).
\] Its objective combines outcome prediction loss and treatment
classification loss:
\[\mathcal{L}_{DragonNet} = \frac{1}{n}\sum_{i = 1}^{n}\lbrack\ell\!\left( Y_{i},\text{ }T_{i}{\widehat{\mu}}_{1}(X_{i}) + (1 - T_{i}){\widehat{\mu}}_{0}(X_{i}) \right) + \alpha\text{ }\ell\!\left( T_{i},\widehat{e}(X_{i}) \right)\rbrack,\]

where \(\ell_{t}\) is the treatment cross-entropy loss,
\[\ell\!_{t}(T_{i},\widehat{e}(X_{i})) = - \lbrack T_{i}\log\widehat{e}(X_{i}) + (1 - T_{i})\log(1 - \widehat{e}(X_{i}))\rbrack.
\]

While this shared-representation strategy can improve statistical
efficiency, the treatment and outcome tasks are still optimized jointly
through the same latent representation. Since the treatment
classification loss is included in the joint objective, gradient updates
will pass information from outcome components to the treatment
components. From a causal modeling perspective, this inverse flow should
be avoided.

\subsection{Tree based methods}
Bayesian Additive Regression Trees (BART) is a nonparametric Bayesian ensemble method that models the response as a sum of many weak regression trees \citep{chipman2010bart}. The outcome model is represented as

\[
Y_i = \sum_{m=1}^{M} g(X_i; T_m, \mu_m) + \varepsilon_i,
\qquad
\varepsilon_i \sim N(0,\sigma^2),
\]

where $T_m$ denotes the structure of the $m$-th tree and $\mu_m$ represents its terminal node parameters. BART is widely used in observational studies because it can flexibly capture nonlinear relationships and complex interactions without requiring explicit model specification.

Causal Forest extends Random Forest methodology to estimate heterogeneous treatment effects directly. The method partitions the covariate space into local neighborhoods and estimates treatment effects within each leaf node \citep{wager2018causalforest}.

By averaging treatment effects across many randomized trees, Causal Forest provides a flexible nonparametric estimator of conditional average treatment effects (CATEs) while reducing overfitting through sample splitting and honesty constraints.

\subsection{Transformer for Tabular Data}

Transformer and attention-based architectures have recently shown strong
performance in tabular modeling because they can flexibly capture
nonlinear dependencies and higher-order feature interactions \citep{wang2025}. 
Such flexibility is potentially valuable for causal inference, however,
the application of transformers in causal inference is still a new
attempt. 

For example, utilizing Transformers to capture complex,
long-range dependencies among time-varying confounders \citep{melnychuk2022}. Do-Probabilistic Foundation Networks (Do\_PFN) is a transformer-based foundation model designed for causal effect estimation. Instead of learning from a single dataset, Do\_PFN is pre-trained on a large collection of synthetically generated causal inference tasks and subsequently applied to new datasets without extensive task-specific retraining \citep{robertson2026dopfn}.

\subsection{Variance estimation for MOCA using Rubin-style pooling.}
In the real data analysis, machine learning based methods like TARNet, and DragonNet mainly provide point estimates and do not natively output variance estimates. In our work, we adopted a Rubin-style pooling rule \citep{barnard1999} to quantify uncertainty for the proposed MOCA estimator. For the $j$-th MOCA run, we defined the predicted CATE for subject $i$ as
\[
\hat{\tau}_{ij}(X_i) = \hat{\mu}_{1,j}(X_i) - \hat{\mu}_{0,j}(X_i),
\]
where $\hat{\mu}_{1,j}(X_i)$ and $\hat{\mu}_{0,j}(X_i)$ are the model-predicted potential outcomes under treatment and control, respectively. The corresponding ATE estimate from the $j$-th run was then computed as the sample mean of the predicted CATE, \(\hat{\tau}_j = \frac{1}{n}\sum_{i=1}^n \hat{\tau}_{ij}.\) To quantify the within-run uncertainty, we estimated the variance of $\hat{\tau}_j$ using the sample variance of the predicted CATE,
\[
U_j
=
\widehat{\mathrm{Var}}(\hat{\tau}_j)
=
\frac{1}{n}\cdot \frac{1}{n-1}\sum_{i=1}^n (\hat{\tau}_{ij} - \bar{\tau}_j)^2,
\] 

where \(\bar{\tau}_j = \frac{1}{n}\sum_{i=1}^n \hat{\tau}_{ij}.\)
Thus, $U_j$ directly reflects the variability of the predicted treatment effects within the $j$-th fitted MOCA model. After repeating MOCA fitting $m$ times, yielding $(\hat{\tau}_1, U_1), \ldots, (\hat{\tau}_m, U_m)$, we combined the results using a Rubin-style pooling rule. Specifically, the pooled treatment effect estimate was \(
\bar{\tau} = \frac{1}{m}\sum_{j=1}^m \hat{\tau}_j,
\) the average within-run variance was
\(
\bar{U} = \frac{1}{m}\sum_{j=1}^m U_j,
\)
and the between-run variance was
\[B = \frac{1}{m-1}\sum_{j=1}^m (\hat{\tau}_j - \bar{\tau})^2.
\]
The final pooled variance was then calculated as
\[
T = \bar{U} + \left(1 + \frac{1}{m}\right)B.
\]
In our real data analysis, we also applied this pooling process for other machine learning methods.

\section{Simulation Setting}

The data generation mechanisms are based on a classical potential-outcome
framework. For each individual \(i = 1,\ldots,n\), let
\(X_i = (X_{i1},\ldots,X_{ip})^\top\) denote the covariate vector,
\(T_i \in \{0,1\}\) the treatment assignment, and
\(Y_i(0),Y_i(1)\) the potential outcomes. The observed outcome is
\[
Y_i = T_iY_i(1) + (1-T_i)Y_i(0).
\]
Across all scenarios, treatment is generated from
\[
T_i \sim \mathrm{Bernoulli}(e_i), \qquad
e_i = \Pr(T_i=1\mid X_i)=\mathrm{clip}\{\mathrm{expit}(\eta_i),0.03,0.97\},
\]
where \(\mathrm{expit}(x)=1/(1+\exp(-x))\). Potential outcomes are
generated as
\[
Y_i(0)=\mu_0(X_i)+\varepsilon_{i0}, \qquad
Y_i(1)=\mu_1(X_i)+\varepsilon_{i1}, \qquad
\mu_1(X_i)=\mu_0(X_i)+\tau(X_i), \qquad
\varepsilon_{i0},\varepsilon_{i1}\overset{i.i.d.}{\sim}N(0,1).
\]

Thus, the true conditional average treatment effect is\(\tau(X_i)=\mu_1(X_i)-\mu_0(X_i).\)
Unless otherwise specified, observed covariates are generated independently
from standard normal distributions:

\paragraph{Linear}
In the linear scenario, \(p=5\). The treatment propensity is specified by
\[
\eta_i
=
-0.15
+2.0X_{i1}
-1.8X_{i2}
+1.5X_{i3}
+1.2X_{i4}
-1.0X_{i5}.
\]
The baseline outcome mean and treatment effect are
\[
\mu_0(X_i)
=
1.0
+0.8X_{i1}
-0.4X_{i2}
+0.6X_{i3}
+0.4X_{i4}
-0.2X_{i5},
\]
\[
\tau(X_i)
=
1.0
+1.6X_{i1}
-1.3X_{i2}
+1.1X_{i3}
+0.9X_{i4}
-0.7X_{i5}.
\]

\paragraph{Linear-Omis}
The Linear-Omis scenario is constructed to be identical to the linear
scenario, except that the outcome additionally depends on an unobserved
variable \(U_i\). Specifically, \(p=5\),
\[
X_{ij} \overset{\mathrm{i.i.d.}}{\sim} N(0,1),
\qquad
U_i \sim N(0,1),
\qquad
U_i \perp X_i.
\]
The variable \(U_i\) is used in the data generating process but is not
observed by any estimator.

The treatment process remains the same as in the linear scenario. The baseline outcome mean conditional on \(X_i\) and \(U_i\) is
\[
\mu_0^{\mathrm{Omis}}(X_i,U_i)
=
1.0
+0.8X_{i1}
-0.4X_{i2}
+0.6X_{i3}
+0.4X_{i4}
-0.2X_{i5}
+U_i.
\]
Thus, the potential outcomes are generated as
\[
Y_i(0)
=
\mu_0^{\mathrm{Omis}}(X_i,U_i)
+\varepsilon_{i0},
\]
\[
Y_i(1)
=
\mu_0^{\mathrm{Omis}}(X_i,U_i)
+\tau(X_i)
+\varepsilon_{i1},
\]
where
\[
\varepsilon_{i0},\varepsilon_{i1}
\overset{\mathrm{i.i.d.}}{\sim}
N(0,\sigma^2).
\]
Equivalently,
\[
Y_i(t)
=
\mu_0(X_i)
+t\,\tau(X_i)
+U_i
+\varepsilon_{it},
\qquad t\in\{0,1\}.
\]

\paragraph{Nonlinear}
In the nonlinear scenario, \(p=8\). Define
\[
s_{1i}=\sin(1.6X_{i1}), \qquad
s_{2i}=\tanh(1.2X_{i2}X_{i3}),
\]
\[
s_{3i}=\mathrm{sign}(X_{i4})\log(1+X_{i4}^2), \qquad
s_{4i}=\tanh(1.5X_{i5}),
\]
\[
h_i=\sin(X_{i1}X_{i3})+0.5\cos(X_{i2}-X_{i5}),
\]
and
\[
\ell_i=1.5X_{i6}-1.2X_{i7}+0.8X_{i8}.
\]
The treatment propensity is generated from
\[
\eta_i
=
-0.05
+3.0s_{1i}
-2.6s_{2i}
+2.3s_{3i}
+2.0s_{4i}
+1.2h_i
+0.6\ell_i.
\]
The baseline outcome mean and treatment effect are
\[
\mu_0(X_i)
=
1.0
+0.25\cos(X_{i1})
+0.18X_{i2}
-0.18X_{i3}X_{i5}
+0.15(X_{i4}^2-1)
+1.2X_{i6}
-0.8X_{i7},
\]
\[
\tau(X_i)
=
1.0
+1.7s_{1i}
-1.3s_{2i}
+1.1s_{3i}
+0.9s_{4i}
+0.7h_i
+\ell_i.
\]
This design introduces nonlinear main effects and interactions in both
the treatment and outcome models.

\paragraph{Hidden confounding}
In this scenario, \(p=8\), and an unobserved confounder
\(U_i\sim N(0,1)\) affects treatment assignment and potential outcomes.
Define
\[
a_i
=
1.6X_{i1}
-1.3X_{i2}
+1.1X_{i3}
+0.9X_{i4}
-0.7X_{i5},
\]
\[
b_i
=
1.2X_{i6}
-0.8X_{i7}
+0.6X_{i8}.
\]
The treatment propensity is
\[
\eta_i
=
-0.15
+1.1a_i
+0.5b_i
+1.2U_i.
\]
The baseline outcome mean and treatment effect are
\[
\mu_0(X_i,U_i)
=
1.0
+0.7X_{i1}
-0.3X_{i2}
+0.5X_{i3}
+0.4X_{i4}
+0.8X_{i6}
-0.5X_{i7}
+1.5U_i,
\]
\[
\tau(X_i,U_i)
=
1.0+a_i+b_i+0.8U_i.
\]
Because \(U_i\) is omitted from the observed covariates used for
estimation, the ignorability assumption is violated.

\paragraph{Hidden confounding (multiple-U)}
Here, multiple latent confounders affect both treatment and outcomes:
\[
U_{i1},U_{i2},U_{i3}\overset{i.i.d.}{\sim}N(0,1).
\]
Using the same \(a_i\) and \(b_i\) as above, the treatment propensity is
\[
\eta_i
=
-0.15
+1.1a_i
+0.5b_i
+1.2U_{i1}
+0.8U_{i2}
-0.6U_{i3}.
\]
The baseline outcome mean and treatment effect are
\[
\mu_0(X_i,U_i)
=
1.0
+0.7X_{i1}
-0.3X_{i2}
+0.5X_{i3}
+0.4X_{i4}
+0.8X_{i6}
-0.5X_{i7}
+1.5U_{i1}
+1.0U_{i2}
-0.5U_{i3},
\]
\[
\tau(X_i,U_i)
=
1.0+a_i+b_i+0.8U_{i1}+0.5U_{i2}-0.3U_{i3}.
\]
Again, ignorability is violated, but more severely than in the
single-\(U\) case.

\paragraph{Heavy-tailed covariates}
This scenario has the same structural form as the linear scenario, but
the covariates are generated from a heavy-tailed distribution:
\[
X_{ij}\overset{i.i.d.}{\sim}t_3,\qquad j=1,\ldots,5.
\]
The treatment propensity is
\[
\eta_i
=
-0.2
+1.8X_{i1}
-1.6X_{i2}
+1.3X_{i3}
+1.0X_{i4}
-0.9X_{i5}.
\]
The baseline outcome mean and treatment effect are
\[
\mu_0(X_i)
=
1.0
+0.6X_{i1}
-0.3X_{i2}
+0.5X_{i3}
+0.3X_{i4}
-0.2X_{i5},
\]
\[
\tau(X_i)
=
1.0
+1.4X_{i1}
-1.1X_{i2}
+0.9X_{i3}
+0.7X_{i4}
-0.6X_{i5}.
\]
This setting preserves a relatively simple structural form while making
estimation harder through heavy-tailed covariates.

\paragraph{High-dimensional}
In the high-dimensional scenario, \(p=300\), with
\[
X_{ij}\overset{i.i.d.}{\sim}N(0,1),\qquad j=1,\ldots,300.
\]
Define
\[
a_i
=
1.8X_{i1}
-1.6X_{i2}
+1.4X_{i3}
+1.1X_{i4}
-0.9X_{i5},
\]
\[
b_i
=
1.2X_{i6}
-1.0X_{i7}
+0.8X_{i8}
+0.6X_{i9}
-0.5X_{i10}.
\]
Let
\[
\beta_0,\beta_\tau \sim N(0,0.1^2I_{300}).
\]
The treatment propensity is
\[
\eta_i=-1.0+a_i+0.4b_i.
\]
The baseline outcome mean and treatment effect are
\[
\mu_0(X_i)
=
1.0+0.4a_i+0.3b_i+X_i^\top\beta_0,
\]
\[
\tau(X_i)
=
1.0+a_i+0.6b_i+X_i^\top\beta_\tau.
\]
This scenario is designed to evaluate performance when the covariate
space is high-dimensional and the outcome surface is driven by a mixture
of sparse and dense signals.

\section{Ablation experiments}

\subsection{Ablation variants}
\paragraph{Remove Linear Branch}
In this variant, the linear branch is removed, and outcome prediction relies solely on the self-attention representation. This experiment evaluates whether explicitly modeling low-order additive effects improves treatment effect estimation. 
\paragraph{Remove Treatment Branch}
The treatment module is completely removed. The outcome module receives only the covariate representation produced by its own encoder and no longer accesses treatment-related representations through cross-attention. This variant tests whether treatment representations learned from the propensity model provide useful information beyond the observed covariates themselves.
\paragraph{Concatenation Fusion}
MOCA combines the outcome representation and treatment representation using gated fusion. In this variant, the gating mechanism is replaced by simple concatenation. This experiment evaluates whether adaptive feature weighting through gated fusion is necessary, or whether a simpler fusion strategy is sufficient.
\paragraph{One-way Cross-attention without Cutting Feedback}
The architecture retains one-way cross-attention from the outcome module to the treatment module, but gradients from the outcome loss are allowed to update treatment-module parameters. This variant isolates the effect of the cutting-feedback mechanism. It evaluates whether preventing outcome information from influencing treatment representation learning improves causal estimation.
\paragraph{Two-way Cross-attention with Cutting Feedback}
Cross-attention is applied in both directions between treatment and outcome. However, gradient blocking is still imposed when information flows from the outcome module back to the treatment module. This variant investigates whether violating the intended one-way information flow affects causal estimation even when gradient contamination is prevented.
\paragraph{Two-way Cross-attention without Cutting Feedback}
Cross-attention is applied in both directions and all gradients are propagated freely throughout the network. This variant represents the strongest violation of the modular causal design.
\subsection{Ablation simulation setting}
The ablation experiments were conducted under a nonlinear simulation setting.
For each replication, we generated
\(n_{\mathrm{train}}=100\), \(n_{\mathrm{valid}}=100\), and
\(n_{\mathrm{test}}=100\) samples, with \(R=50\) independent replications.
The covariate dimension was set to \(p=8\), and covariates were generated as
\(X_{ij}\overset{i.i.d.}{\sim}N(0,1)\). Treatment assignment followed
\[
T_i\sim \mathrm{Bernoulli}(e_i),\qquad
e_i=\mathrm{clip}\{\mathrm{expit}(\eta_i),0.03,0.97\}.
\]
Define
\[
s_{1i}=\sin(1.6X_{i1}),\quad
s_{2i}=\tanh(1.2X_{i2}X_{i3}),
\]
\[
s_{3i}=\mathrm{sign}(X_{i4})\log(1+X_{i4}^2),\quad
s_{4i}=\tanh(1.5X_{i5}),
\]
\[
h_i=\sin(X_{i1}X_{i3})+0.5\cos(X_{i2}-X_{i5}),
\]
and
\[
\ell_i=1.5X_{i6}-1.2X_{i7}+0.8X_{i8}.
\]
The treatment assignment mechanism was
\[
\eta_i
=
-0.05
+3.0s_{1i}
-2.6s_{2i}
+2.3s_{3i}
+2.0s_{4i}
+1.2h_i
+0.6\ell_i.
\]
The baseline outcome and conditional treatment effect were generated by
\[
\mu_0(X_i)
=
1.0
+0.25\cos(X_{i1})
+0.18X_{i2}
-0.18X_{i3}X_{i5}
+0.15(X_{i4}^2-1)
+1.2X_{i6}
-0.8X_{i7},
\]
and
\[
\tau(X_i)
=
1.0
+1.7s_{1i}
-1.3s_{2i}
+1.1s_{3i}
+0.9s_{4i}
+0.7h_i
+\ell_i.
\]
Finally,
\[
\mu_1(X_i)=\mu_0(X_i)+\tau(X_i),
\]
and the potential outcomes were generated as
\[
Y_i(0)=\mu_0(X_i)+\varepsilon_{i0},\qquad
Y_i(1)=\mu_1(X_i)+\varepsilon_{i1},
\]
where \(\varepsilon_{i0},\varepsilon_{i1}\overset{i.i.d.}{\sim}N(0,1)\).
The observed outcome was
\[
Y_i=T_iY_i(1)+(1-T_i)Y_i(0).
\]

\section{Theory for MOCA}
\subsection{Markov-kernel modularity and Gaussian predictive KL optimality}
\label{subsec:kernel-kl}

A Markov kernel assigns each input in one measurable space a probability measure
on another measurable space, with measurable dependence on the input
\citep{Kallenberg1997}. Markov kernels can be composed as directed stochastic
mappings, and provides an abstract framework for such
compositions and conditional products \citep{fritz2020synthetic}. In the present setting,
learning kernels describe the distribution of representations produced by
deterministic or randomized training algorithms conditional on the observed data.
Their sequential composition provides a precise way to encode directed information
flow: the treatment representation is learned first, and the outcome representation
is then learned conditional on the fixed treatment representation.

A related but distinct causal perspective is the algorithmic Markov condition of
Janzing and Sch{\"o}lkopf, which motivates causal direction through algorithmic
independence between causal mechanisms rather than through probability kernels
\citep{janzing2010causal}. Motivated by cutting feedback in modular Bayesian
analysis \citep{liu2025,frazier2025}, MOCA uses the kernel construction to
define cutting feedback at the representation level. The development below proceeds
by defining the learning kernels, constructing and validating their directed product,
establishing treatment autonomy, and connecting this factorization to staged
optimization and Gaussian predictive KL optimality.

\subsubsection{Data modules, representations, and learning kernels}

Let \(D_T=\{(X_i,T_i)\}_{i=1}^n\) denote the treatment-side data, and let \(D_O=\{Y_i\}_{i=1}^n\)
denote the additional outcome-side information. The full observed dataset is
\(D=(D_T,D_O).\)

Let $Z_T\in\mathcal{Z}_T$ denote the treatment representation and $Z_Y\in\mathcal{Z}_Y$ denote the outcome representation. Assume that all relevant measurable spaces are standard Borel spaces. Define the treatment learning kernel
\(K_T(dz_T\mid D_T)\)
and the outcome learning kernel \(K_Y(dz_Y\mid D_T,D_O,z_T)\).
The first kernel describes the distribution of the learned treatment representation conditional on treatment-side data, whereas the second describes the distribution of the learned outcome representation conditional on the observed data and the fixed treatment representation.

\subsubsection{Construction of the MOCA cut kernel}

For measurable sets $A_T\subseteq\mathcal{Z}_T$ and $A_Y\subseteq\mathcal{Z}_Y$, define the cut learning kernel by
\[K_{\mathrm{cut}}(A_T\times A_Y\mid D_T,D_O)
=
\int_{A_T}K_Y(A_Y\mid D_T,D_O,z_T)K_T(dz_T\mid D_T).
\]
In differential notation,
\begin{equation}
\label{eq:cut-kernel-density}
K_{\mathrm{cut}}(dz_T,dz_Y\mid D_T,D_O)
=
K_T(dz_T\mid D_T)K_Y(dz_Y\mid D_T,D_O,z_T).
\end{equation}
This is a directed conditional product and represents the allowed direction
\(Z_T\longrightarrow Z_Y\).

\begin{proposition}[Validity of the cut kernel]
\label{prop:cut-kernel-validity}
For each fixed $(D_T,D_O)$, $K_{\mathrm{cut}}(\cdot\mid D_T,D_O)$ is a probability measure on $\mathcal{Z}_T\times\mathcal{Z}_Y$. Moreover, for every measurable set $A\subseteq\mathcal{Z}_T\times\mathcal{Z}_Y$, the map
\[
(D_T,D_O)\longmapsto K_{\mathrm{cut}}(A\mid D_T,D_O)
\]
is measurable. Hence, $K_{\mathrm{cut}}$ is a valid Markov kernel.
\end{proposition}

\begin{proof}
First consider a measurable rectangle $A_T\times A_Y$. Since $K_Y$ is a Markov kernel,
\[
0\leq K_Y(A_Y\mid D_T,D_O,z_T)\leq 1,
\]
and therefore
\[
0\leq K_{\mathrm{cut}}(A_T\times A_Y\mid D_T,D_O)\leq 1.
\]
When $A_T=\mathcal{Z}_T$ and $A_Y=\mathcal{Z}_Y$,
\[
K_Y(\mathcal{Z}_Y\mid D_T,D_O,z_T)=1,
\]
\begin{align*}
K_{\mathrm{cut}}(\mathcal{Z}_T\times\mathcal{Z}_Y\mid D_T,D_O)
&=\int_{\mathcal{Z}_T}K_Y(\mathcal{Z}_Y\mid D_T,D_O,z_T)K_T(dz_T\mid D_T)\\
&=\int_{\mathcal{Z}_T}K_T(dz_T\mid D_T)=1.
\end{align*}
Countable additivity follows from the countable additivity of $K_T$ and $K_Y$, together with Tonelli's theorem. Because $(D_T,D_O,z_T)\mapsto K_Y(A_Y\mid D_T,D_O,z_T)$ is measurable and $K_T$ is itself a Markov kernel,
\[
(D_T,D_O)\mapsto\int_{A_T}K_Y(A_Y\mid D_T,D_O,z_T)K_T(dz_T\mid D_T)
\]
is measurable. Extension from measurable rectangles to the product $\sigma$-algebra establishes that $K_{\mathrm{cut}}$ is a valid Markov kernel.
\end{proof}

\subsubsection{Representation-level cutting feedback}

\begin{definition}[Treatment autonomy]
\label{def:treatment-autonomy}
A learning kernel satisfies treatment autonomy if, for every measurable set $A_T\subseteq\mathcal{Z}_T$,
\[
K(A_T\times\mathcal{Z}_Y\mid D_T,D_O)=K_T(A_T\mid D_T).
\]
Equivalently,
\[K_T(dz_T\mid D_T,D_O)=K_T(dz_T\mid D_T).
\]
\end{definition}

\begin{proposition}[MOCA satisfies treatment autonomy]
\label{prop:moca-autonomy}
If the MOCA learning kernel factorizes as in \eqref{eq:cut-kernel-density}, then
\[
K_{\mathrm{cut}}(A_T\times\mathcal{Z}_Y\mid D_T,D_O)=K_T(A_T\mid D_T).
\]
\end{proposition}

\begin{proof}
Set $A_Y=\mathcal{Z}_Y$. Since $K_Y(\mathcal{Z}_Y\mid D_T,D_O,z_T)=1$,
\begin{align*}
K_{\mathrm{cut}}(A_T\times\mathcal{Z}_Y\mid D_T,D_O)
&=\int_{A_T}K_Y(\mathcal{Z}_Y\mid D_T,D_O,z_T)K_T(dz_T\mid D_T)\\
&=\int_{A_T}K_T(dz_T\mid D_T)\\
&=K_T(A_T\mid D_T).
\end{align*}
The marginal law of $Z_T$ therefore does not depend on $D_O$.
\end{proof}

This establishes representation-level cutting feedback. Outcome learning may use the treatment representation, but outcome-side optimization cannot alter that representation.

\subsubsection{Connection between MOCA training and the cut-kernel factorization}

MOCA's staged optimization procedure provides an algorithmic realization of the cut-kernel factorization. Let $\theta_T$ and $\theta_Y$ denote the parameters of the treatment and outcome modules, respectively. The treatment module is first learned exclusively from treatment-side information,
\[
\widehat{\theta}_T=\mathcal{A}_T(D_T,U_T),
\]
and is subsequently held fixed during outcome optimization,
\[
\widehat{\theta}_Y=\mathcal{A}_Y(D_T,D_O,\widehat{\theta}_T,U_Y).
\]
Consequently,
\[
\mathcal{L}(\widehat{\theta}_T\mid D_T,D_O)=\mathcal{L}(\widehat{\theta}_T\mid D_T),
\]
which implies
\[
K_T(dz_T\mid D_T,D_O)=K_T(dz_T\mid D_T).
\]
Thus, staged optimization with gradient detachment connects MOCA's practical training procedure to treatment autonomy in the cut-kernel formulation.

\subsubsection{Gaussian predictive kernels}

Fix the learned treatment representation and write
\[
Z_T=h_T(X),\qquad W=(X,Z_T).
\]
For treatment arm $t\in\{0,1\}$, denote the true conditional outcome law by
\[
P_t(dy\mid w)=\mathcal{L}(Y\mid W=w,T=t).
\]
Assume that $P_t(\cdot\mid w)$ admits a density $p_t(y\mid w)$ and has a finite second moment,
\[
\mathbb{E}(Y^2\mid W=w,T=t)<\infty.
\]
For any prediction function $f_t$, define the fixed-variance Gaussian predictive kernel
\[
Q_{f,t}(dy\mid w)=\mathcal{N}\!\left(dy;f_t(w),\sigma_t^2\right),
\qquad \sigma_t^2>0,
\]
with density
\[
q_{f,t}(y\mid w)
=
\frac{1}{\sqrt{2\pi\sigma_t^2}}
\exp\!\left[-\frac{\{y-f_t(w)\}^2}{2\sigma_t^2}\right].
\]
Define the expected conditional forward-KL risk as
\[
\mathcal{K}_t(f_t)
=
\mathbb{E}_{W\mid T=t}
\left[
\operatorname{KL}\{P_t(\cdot\mid W)\,\|\,Q_{f,t}(\cdot\mid W)\}
\right].
\]

\begin{theorem}[Gaussian predictive KL--MSE equivalence]
\label{thm:supp-kl-mse}
Assume that the expected conditional KL divergence is finite. Then there exists a constant $C_t$, independent of $f_t$, such that
\begin{equation}
\label{eq:kl-mse}
\mathcal{K}_t(f_t)
=
C_t+\frac{1}{2\sigma_t^2}
\mathbb{E}\!
\left[
\{Y-f_t(W)\}^2\mid T=t
\right].
\end{equation}
Therefore, for any function class $\mathcal{F}_t$,
\[
\arg\min_{f_t\in\mathcal{F}_t}\mathcal{K}_t(f_t)
=
\arg\min_{f_t\in\mathcal{F}_t}
\mathbb{E}\!\left[\{Y-f_t(W)\}^2\mid T=t\right].
\]
\end{theorem}

\begin{proof}
Conditional on $W=w$,
\begin{align*}
&\operatorname{KL}\{P_t(\cdot\mid w)\,\|\,Q_{f,t}(\cdot\mid w)\}\\
&\quad=\int p_t(y\mid w)\log\frac{p_t(y\mid w)}{q_{f,t}(y\mid w)}\,dy\\
&\quad=\int p_t(y\mid w)\log p_t(y\mid w)\,dy
-\int p_t(y\mid w)\log q_{f,t}(y\mid w)\,dy.
\end{align*}
The Gaussian negative log density is
\[
-\log q_{f,t}(y\mid w)
=
\frac{1}{2}\log(2\pi\sigma_t^2)
+
\frac{\{y-f_t(w)\}^2}{2\sigma_t^2}.
\]
Hence,
\begin{align*}
&\operatorname{KL}\{P_t(\cdot\mid w)\,\|\,Q_{f,t}(\cdot\mid w)\}\\
&\quad=C_t(w)
+
\frac{1}{2\sigma_t^2}
\mathbb{E}\!\left[\{Y-f_t(w)\}^2\mid W=w,T=t\right],
\end{align*}
where
\[
C_t(w)
=
\int p_t(y\mid w)\log p_t(y\mid w)\,dy
+
\frac{1}{2}\log(2\pi\sigma_t^2)
\]
does not depend on $f_t$. Taking expectation over $W\mid T=t$ gives \eqref{eq:kl-mse}, with $C_t=\mathbb{E}\{C_t(W)\mid T=t\}$. Since $1/(2\sigma_t^2)>0$, the KL and MSE objectives have identical minimizers.
\end{proof}

\subsubsection{Gaussian predictive interpretation of the MOCA outcome loss}

MOCA is trained using the squared-error loss $\{Y-f_t(W)\}^2$. For a fixed-variance Gaussian predictive model,
\[
-\log q_{f,t}(Y\mid W)
=
\frac{1}{2\sigma_t^2}\{Y-f_t(W)\}^2+\text{constant}.
\]
Therefore, the Gaussian predictive kernel is the most direct probabilistic representation of the MSE objective used by MOCA.

The ATE depends only on the two conditional mean functions,
\[
\psi=\mathbb{E}\{\mu_1(X)-\mu_0(X)\},
\qquad
\mu_t(X)=\mathbb{E}\{Y(t)\mid X\}.
\]
The location parameter of the Gaussian predictive family is the prediction function $f_t$, and its KL-optimal location parameter is the conditional mean. Thus, the Gaussian predictive interpretation is directly aligned with the quantity required for plug-in ATE estimation.

\subsection{MSE targets, best $L^2$ approximation, and ATE theory}
\label{subsec:mse-ate}

\subsubsection{Population masked MSE}

Fix $Z_T=h_T(X)$ and define $W=(X,Z_T)$. The population outcome loss of MOCA is
\[\mathcal{R}_Y(f_0,f_1)
=
\mathbb{E}
\left[
T\{Y-f_1(W)\}^2+(1-T)\{Y-f_0(W)\}^2
\right].
\]
Let $\pi_t=P(T=t)$. Then
\[
\mathcal{R}_Y(f_0,f_1)
=
\pi_1\mathcal{R}_1(f_1)+\pi_0\mathcal{R}_0(f_0),
\]
where
\[
\mathcal{R}_t(f_t)
=
\mathbb{E}\left[\{Y-f_t(W)\}^2\mid T=t\right].
\]
Because $f_0$ and $f_1$ appear in separate terms, the two treatment arms may be analyzed independently.

\subsubsection{Conditional-mean decomposition}

Define
\[
m_t(W)=\mathbb{E}(Y\mid W,T=t).
\]

\begin{theorem}[The conditional mean minimizes population MSE]
\label{thm:conditional-mean-mse}
For every square-integrable function $f_t$,
\begin{equation}
\label{eq:mse-decomposition}
\mathcal{R}_t(f_t)
=
\mathcal{R}_t(m_t)
+
\mathbb{E}\left[\{f_t(W)-m_t(W)\}^2\mid T=t\right].
\end{equation}
Therefore,
\[
m_t\in\arg\min_f\mathcal{R}_t(f)
\]
up to almost-sure equivalence.
\end{theorem}

\begin{proof}
Write
\[
Y-f_t(W)=Y-m_t(W)+m_t(W)-f_t(W).
\]
After squaring and conditioning on $(W,T=t)$, the cross term vanishes because
\[
\mathbb{E}\{Y-m_t(W)\mid W,T=t\}=0.
\]
Thus,
\begin{align*}
&\mathbb{E}\left[\{Y-f_t(W)\}^2\mid W,T=t\right]\\
&\quad=
\mathbb{E}\left[\{Y-m_t(W)\}^2\mid W,T=t\right]
+
\{m_t(W)-f_t(W)\}^2.
\end{align*}
Taking expectation over $W\mid T=t$ yields \eqref{eq:mse-decomposition}. The second term is nonnegative and is zero if and only if $f_t(W)=m_t(W)$ almost surely.
\end{proof}

\subsubsection{Best $L^2$ approximation within the MOCA function class}

Let $\mathcal{F}_t$ denote the function class represented by the MOCA outcome head.

\begin{theorem}[Best $L^2$ approximation]
\label{thm:best-l2}
Suppose a minimizer exists. Then
\[
f_{t,\mathcal{F}}^*\in\arg\min_{f_t\in\mathcal{F}_t}\mathcal{R}_t(f_t)
\]
if and only if
\[
f_{t,\mathcal{F}}^*
\in
\arg\min_{f_t\in\mathcal{F}_t}
\|f_t-m_t\|_{L^2(P_{W\mid T=t})}^2.
\]
\end{theorem}

\begin{proof}
By Theorem~\ref{thm:conditional-mean-mse},
\[
\mathcal{R}_t(f_t)
=
\mathcal{R}_t(m_t)
+
\|f_t-m_t\|_{L^2(P_{W\mid T=t})}^2.
\]
The first term does not depend on $f_t$. Therefore, minimizing $\mathcal{R}_t(f_t)$ over $\mathcal{F}_t$ is equivalent to minimizing the squared $L^2$ distance from $m_t$.
\end{proof}

Thus, the statistically precise statement is that MOCA learns the best $L^2$ approximation to the conditional mean within its function class.

\subsubsection{Ideal ATE recovery}

Under consistency and conditional ignorability, and because $Z_T$ is a measurable function of $X$ while the outcome module retains $X$, the population MSE target satisfies
\[\mathbb{E}(Y\mid X,Z_T,T=t)
=
\mathbb{E}(Y\mid X,T=t)
=
\mathbb{E}\{Y(t)\mid X\}
=
\mu_t(X).
\]
Thus, the treatment-specific conditional means learned by MOCA coincide with the potential-outcome regression functions under the standard causal identification assumptions.

The true ATE is
\[
\psi=\mathbb{E}\{Y(1)-Y(0)\}.
\]
Under consistency, conditional ignorability, and positivity,
\[
\psi=\mathbb{E}_X\{\mu_1(X)-\mu_0(X)\}.
\]
For a pair of outcome functions $f=(f_0,f_1)$, define the population plug-in functional
\[
\psi_f
=
\mathbb{E}_X
\left[
 f_1\{X,h_T(X)\}-f_0\{X,h_T(X)\}
\right].
\]

\begin{theorem}[Ideal ATE recovery by the population MOCA outcome model]
\label{thm:ideal-ate}
Assume that:
\begin{enumerate}
\renewcommand{\labelenumi}{(\roman{enumi})}
\item consistency, conditional ignorability, and positivity hold;
\item the treatment representation satisfies $Z_T=h_T(X)$, and the outcome predictor conditions on $X$, possibly together with $Z_T$;
\item the true potential-outcome regression functions satisfy $\mu_t\in\mathcal{F}_t$ for $t\in\{0,1\}$;
\item the population masked MSE admits a minimizer; and
\item MOCA attains a population minimizer in each treatment arm.
\end{enumerate}
Then
\[
f_t^*(X,Z_T)=\mu_t(X),\qquad t\in\{0,1\},
\]
and consequently
\[
\psi_{\mathrm{MOCA}}^*
=
\mathbb{E}\{f_1^*(X,Z_T)-f_0^*(X,Z_T)\}
=
\psi.
\]
\end{theorem}

\begin{proof}
By Theorem~\ref{thm:conditional-mean-mse},
\[
f_t^*(X,Z_T)=\mathbb{E}(Y\mid X,Z_T,T=t).
\]
Because $Z_T=h_T(X)$,
\[
\mathbb{E}(Y\mid X,Z_T,T=t)=\mathbb{E}(Y\mid X,T=t).
\]
By consistency and conditional ignorability,
\[
\mathbb{E}(Y\mid X,T=t)=\mathbb{E}\{Y(t)\mid X\}=\mu_t(X).
\]
Substitution into the plug-in functional gives
\begin{align*}
\psi_{\mathrm{MOCA}}^*
&=\mathbb{E}_X\{f_1^*(X,Z_T)-f_0^*(X,Z_T)\}\\
&=\mathbb{E}_X\{\mu_1(X)-\mu_0(X)\}\\
&=\psi.
\end{align*}
\end{proof}

\subsubsection{From arm-specific $L^2$ approximation to population ATE error}

Let
\[
e_t(X)=f_t\{X,h_T(X)\}-\mu_t(X)
\]
denote the approximation error of the treatment-specific outcome predictor. Then
\[
\psi_f-\psi=\mathbb{E}\{e_1(X)-e_0(X)\}.
\]
By the triangle inequality and Cauchy--Schwarz,
\[
|\psi_f-\psi|
\leq
\|e_1\|_{L^2(P_X)}+\|e_0\|_{L^2(P_X)}.
\]

\begin{proposition}[Population ATE error bound]
\label{prop:ate-error-bound}
The population plug-in ATE error satisfies
\begin{equation}
\label{eq:ate-error-l2}
|\psi_f-\psi|
\leq
\|e_1\|_{L^2(P_X)}+\|e_0\|_{L^2(P_X)},
\end{equation}
and consequently
\[
(\psi_f-\psi)^2
\leq
2\|e_1\|_{L^2(P_X)}^2+2\|e_0\|_{L^2(P_X)}^2.
\]
\end{proposition}

Thus, small potential-outcome $L^2$ error implies small population ATE error. However, MOCA's masked MSE evaluates each outcome predictor under the corresponding treatment-arm distribution $P_{X\mid T=t}$, whereas the ATE is defined under the marginal target distribution $P_X$. To connect these distributions, let
\[
p_t(X)=P(T=t\mid X),
\qquad
\pi_t=P(T=t),
\]
and define the arm-specific excess MSE risk as
\[
\mathcal{E}_t(f_t)
=
\mathcal{R}_t(f_t)-\mathcal{R}_t(m_t)
=
\mathbb{E}\{e_t(X)^2\mid T=t\}.
\]
Assume strong overlap,
\[
p_t(X)\geq\varepsilon_t>0
\]
almost surely. Then
\begin{align*}
\pi_t\mathcal{E}_t(f_t)
&=P(T=t)\mathbb{E}\{e_t(X)^2\mid T=t\}\\
&=\mathbb{E}\{p_t(X)e_t(X)^2\}\\
&\geq\varepsilon_t\mathbb{E}\{e_t(X)^2\}.
\end{align*}
It follows that
\[
\|e_t\|_{L^2(P_X)}
\leq
\sqrt{\frac{\pi_t}{\varepsilon_t}}
\sqrt{\mathcal{E}_t(f_t)}.
\]
Substitution into \eqref{eq:ate-error-l2} gives
\[
|\psi_f-\psi|
\leq
\sum_{t=0}^1
\sqrt{\frac{\pi_t}{\varepsilon_t}}
\sqrt{\mathcal{E}_t(f_t)},
\]
and therefore
\[
(\psi_f-\psi)^2
\leq
2\sum_{t=0}^1
\frac{\pi_t}{\varepsilon_t}
\mathcal{E}_t(f_t).
\]
Hence, under strong overlap, MOCA's arm-specific excess MSE directly controls its population ATE error. The dependence on $\varepsilon_t$ also shows that this guarantee becomes weaker under poor overlap, because prediction accuracy within the observed treatment arms may not extend reliably to the full target population.

\subsubsection{Feedback-induced ATE distortion and risk decomposition}

Let the one-way cut solution produce treatment-specific outcome functions
\[
f_t^{\mathrm{cut}}(X),\qquad t\in\{0,1\}.
\]
Suppose that allowing two-way feedback changes the learned outcome functions to
\[
f_t^{\mathrm{fb}}(X)
=
f_t^{\mathrm{cut}}(X)+\Delta_t(X),
\]
where $\Delta_t(X)$ denotes the arm-specific prediction shift induced by feedback. The population plug-in ATEs under the cut and feedback solutions are
\[
\psi^{\mathrm{cut}}
=
\mathbb{E}\{f_1^{\mathrm{cut}}(X)-f_0^{\mathrm{cut}}(X)\}
\]
and
\[
\psi^{\mathrm{fb}}
=
\mathbb{E}\{f_1^{\mathrm{fb}}(X)-f_0^{\mathrm{fb}}(X)\}.
\]
Define the average feedback-induced distortion as
\[
\delta_{\mathrm{fb}}
=
\mathbb{E}\{\Delta_1(X)-\Delta_0(X)\}.
\]
Then
\[
\psi^{\mathrm{fb}}=\psi^{\mathrm{cut}}+\delta_{\mathrm{fb}}.
\]
Let
\[
b_{\mathrm{cut}}=\psi^{\mathrm{cut}}-\psi
\]
denote the population ATE bias of the cut solution. Define the squared population ATE errors as
\[
R_{\mathrm{cut}}=(\psi^{\mathrm{cut}}-\psi)^2,
\qquad
R_{\mathrm{fb}}=(\psi^{\mathrm{fb}}-\psi)^2.
\]

\begin{theorem}[Feedback risk decomposition]
\label{thm:feedback-risk}
The difference between the feedback and cut ATE errors is
\[R_{\mathrm{fb}}-R_{\mathrm{cut}}
=
2b_{\mathrm{cut}}\delta_{\mathrm{fb}}+\delta_{\mathrm{fb}}^2.
\]
\end{theorem}

\begin{proof}
Since $\psi^{\mathrm{fb}}-\psi=b_{\mathrm{cut}}+\delta_{\mathrm{fb}}$,
\begin{align*}
R_{\mathrm{fb}}-R_{\mathrm{cut}}
&=(b_{\mathrm{cut}}+\delta_{\mathrm{fb}})^2-b_{\mathrm{cut}}^2\\
&=2b_{\mathrm{cut}}\delta_{\mathrm{fb}}+\delta_{\mathrm{fb}}^2.
\end{align*}
\end{proof}

\begin{corollary}[Feedback applied to an ATE-correct cut solution]
\label{cor:feedback-correct-cut}
Suppose that $\psi^{\mathrm{cut}}=\psi$. Then $b_{\mathrm{cut}}=0$ and
\[
R_{\mathrm{fb}}-R_{\mathrm{cut}}
=
\delta_{\mathrm{fb}}^2
=
\left[
\mathbb{E}\{\Delta_1(X)-\Delta_0(X)\}
\right]^2
\geq 0.
\]
Therefore, additional outcome-to-treatment feedback cannot improve an already correct cut ATE solution. The feedback solution has the same ATE error only if
\[
\mathbb{E}\{\Delta_1(X)-\Delta_0(X)\}=0,
\]
in which case the average arm-specific perturbations cancel. Otherwise, $R_{\mathrm{fb}}>R_{\mathrm{cut}}$, and feedback introduces a strictly positive population ATE error.
\end{corollary}

\subsection{Interpretation of the two theoretical components}

The Markov-kernel component explains how MOCA learns. It establishes
\[
K_{\mathrm{MOCA}}(dz_T,dz_Y\mid D_T,D_O)
=
K_T(dz_T\mid D_T)K_Y(dz_Y\mid D_T,D_O,z_T),
\]
Hence, MOCA is a one-way modular learner satisfying treatment autonomy. Conditional on this autonomous treatment representation, the Gaussian predictive result establishes
\[
\arg\min_{f_t}
\operatorname{KL}(P_t\,\|\,Q_{f,t})
=
\arg\min_{f_t}
\operatorname{MSE}(f_t)
\]
within the fixed-variance Gaussian predictive family. Thus, the MSE-trained MOCA outcome head induces a predictively KL-optimal Gaussian kernel within its model class.

The MSE component explains what MOCA learns. It establishes
\[
f_t^*=\mathbb{E}(Y\mid X,Z_T,T=t),
\]
and, under the causal identification assumptions,
\[
f_t^*=\mu_t(X).
\]
Within a restricted function class, MOCA learns the best $L^2$ approximation to the conditional mean. Under strong overlap,
\[
|\psi_f-\psi|
\leq
\sum_{t=0}^1
\sqrt{\frac{\pi_t}{\varepsilon_t}}
\sqrt{\mathcal{E}_t(f_t)},
\]
so arm-specific excess prediction risk explicitly controls population ATE error.

Finally, for an already correct cut solution,
\[
R_{\mathrm{fb}}-R_{\mathrm{cut}}
=
\left[
\mathbb{E}\{\Delta_1(X)-\Delta_0(X)\}
\right]^2.
\]
Therefore, additional feedback cannot improve an already correct cut ATE solution. It preserves the same ATE only when the average arm-specific perturbations cancel exactly; otherwise, it introduces strictly positive population ATE error.

\section{Conformal Inference for Individual Treatment Effects}

In addition to point estimates, we construct uncertainty intervals for individual treatment effects using the conformal counterfactual inference framework of \citet{lei2021conformal}. Let $X_i$ denote the observed covariates, $T_i\in\{0,1\}$ the treatment assignment, and $Y_i=Y_i(T_i)$ the observed outcome. The individual treatment effect is defined as \(\tau_i=Y_i(1)-Y_i(0)\).

Because only one potential outcome is observed for each unit, constructing an interval for $\tau_i$ reduces to constructing a prediction interval for the missing counterfactual outcome. For each treatment arm $t\in\{0,1\}$, the corresponding observations are divided into a proper training set and a calibration set.

The outcome model is fitted on the proper training set to obtain
\(\widehat{\mu}_t(x)\approx \mathbb{E}\bigl[Y(t)\mid X=x\bigr].\)
For each calibration observation $j$ with $T_j=t$, the absolute-residual conformity score is \(R_{j,t}=\left|Y_j-\widehat{\mu}_t(X_j)\right|\).

Standard split conformal prediction calibrates prediction intervals under the covariate distribution $X\mid T=t$. In counterfactual prediction, however, a model calibrated in one treatment arm is applied to units from the opposite arm. Weighted conformal calibration is therefore used to account for this covariate shift. Let \(e(x)=\mathbb{P}(T=1\mid X=x)\) denote the propensity score. Up to a multiplicative normalization constant, the density ratio weights are
\[
w_0(x)=\frac{e(x)}{1-e(x)}
\]
when predicting $Y(0)$ for treated units, and
\[
w_1(x)=\frac{1-e(x)}{e(x)}
\]
when predicting $Y(1)$ for control units. The same propensity score estimator, training calibration splits, and calibration procedure are used for MOCA and all baseline methods.

For a target unit with covariates $x$, let $q_{t,1-\alpha}(x)$ denote the weighted conformal $(1-\alpha)$ quantile of the calibration residuals in arm $t$. The resulting prediction interval for the missing potential outcome is
\[
\widehat{C}_t(x)
=
\left[
\widehat{\mu}_t(x)-q_{t,1-\alpha}(x),
\widehat{\mu}_t(x)+q_{t,1-\alpha}(x)
\right].
\]

The ITE interval is obtained by combining the observed factual outcome with the conformal interval for the missing counterfactual outcome. For a treated unit with $T_i=1$, suppose 
\(
\widehat{C}_0(X_i)
=
\left[\widehat{L}_{i,0},\widehat{U}_{i,0}\right].
\)
Because $\tau_i=Y_i-Y_i(0)$, the corresponding ITE interval is
\(
\widehat{C}_{\tau,i}
=
\left[
Y_i-\widehat{U}_{i,0},
Y_i-\widehat{L}_{i,0}
\right].
\) For a control unit with $T_i=0$, suppose 
\(
\widehat{C}_1(X_i)
=
\left[\widehat{L}_{i,1},\widehat{U}_{i,1}\right].
\)
Because $\tau_i=Y_i(1)-Y_i$, the corresponding ITE interval is 
\(
\widehat{C}_{\tau,i}
=
\left[
\widehat{L}_{i,1}-Y_i,
\widehat{U}_{i,1}-Y_i
\right].
\)

For simulated datasets in which the true ITEs are available, interval quality is evaluated using empirical coverage and average interval length. Empirical coverage is defined as
\[
\operatorname{Coverage}
=
\frac{1}{n_{\mathrm{test}}}
\sum_{i=1}^{n_{\mathrm{test}}}
\mathbb{I}\left\{
\tau_i\in\widehat{C}_{\tau,i}
\right\},
\]
and average interval length is defined as
\[
\operatorname{Length}
=
\frac{1}{n_{\mathrm{test}}}
\sum_{i=1}^{n_{\mathrm{test}}}
\left(
\widehat{U}_{\tau,i}-\widehat{L}_{\tau,i}
\right).
\]
Empirical coverage measures whether the intervals attain the nominal level $1-\alpha$, whereas average interval length measures their statistical efficiency.

\newpage
\section{Pseudocode}

\newcommand{\algstep}[1]{%
  \Statex \hspace{\algorithmicindent}\textbf{#1}
}

\begin{algorithm}[H]
\caption{MOCA: Train treatment module}
\label{alg:mocat}
\begingroup
\fontsize{10}{15}\selectfont
\renewcommand{\baselinestretch}{0.88}\selectfont

\setlength{\abovedisplayskip}{3pt}
\setlength{\belowdisplayskip}{3pt}
\setlength{\abovedisplayshortskip}{2pt}
\setlength{\belowdisplayshortskip}{2pt}

\begin{algorithmic}[1]
\Require Observed data $\mathcal{D}=\{(X_i,T_i)\}_{i=1}^n$
\Ensure Estimated propensity score $\widehat e_i$, treatment latent representation  \(H_{i}^{(T)} = (z_{i}^{(T)})^{\top} \in \mathbb{R}^{1 \times d}.
\)

\State Initialize treatment-module parameters $\theta_T$
\For{each training epoch}
    \For{each mini-batch $\mathcal{B}\subset \mathcal{D}$}

        \algstep{Step 1: Tokenization and Transformer Encoder}
        \State Tokenize covariates $X_i$ into scalar feature tokens.
        \State Construct treatment linear branch:
        \(H_i^{(L,T)}=\mathrm{Tok}^{(L,T)}(X_i).\)
        \State Construct treatment self-attention branch:
        \(H_i^{(S,T)}=\mathrm{Enc}_T\{\mathrm{Tok}^{(S,T)}(X_i)\}.\)
        \State Use treatment query $q^{(T)}$ to summarize both branches:
        \[
        z_i^{(L,T)}=\mathrm{MHA}(q^{(T)},H_i^{(L,T)},H_i^{(L,T)}),
        \]
        \[
        z_i^{(S,T)}=\mathrm{MHA}(q^{(T)},H_i^{(S,T)},H_i^{(S,T)}).
        \]

        \algstep{Step 2: Fusion gate}
        \State Compute treatment gate weights:
        \[
        \alpha_i^{(T)}
        =
        \mathrm{softmax}
        \left(
        \frac{
        g_i^{(T)}([z_i^{(L,T)};z_i^{(S,T)}])
        }{\tau_g}
        \right).
        \]
        \State Form gated treatment representation:
        \[
        z_i^{(T)}
        =
        \alpha_{iL}^{(T)}z_i^{(L,T)}
        +
        \alpha_{iS}^{(T)}z_i^{(S,T)}.
        \]
        \State Estimate propensity score:
        \[
        \widehat e_i
        =
        \sigma
        \left(
        P_T([\alpha_{iL}^{(T)}z_i^{(L,T)};
        \alpha_{iS}^{(T)}z_i^{(S,T)}])
        \right).
        \]
        \State Update $\theta_T$ by minimizing
        \[
        \mathcal{L}_T
        =
        -\frac{1}{|\mathcal{B}|}
        \sum_{i\in\mathcal{B}}
        \left[
        T_i\log \widehat e_i
        +
        (1-T_i)\log(1-\widehat e_i)
        \right].
        \]

    \EndFor
\EndFor
\end{algorithmic}
\endgroup
\end{algorithm}
\newpage
\begin{algorithm}[H]
\caption{MOCA: Train outcome module with cutting feedback}
\label{alg:mocao}
\begingroup
\fontsize{10}{15}\selectfont
\renewcommand{\baselinestretch}{0.88}\selectfont

\setlength{\abovedisplayskip}{3pt}
\setlength{\belowdisplayskip}{3pt}
\setlength{\abovedisplayshortskip}{2pt}
\setlength{\belowdisplayshortskip}{2pt}

\begin{algorithmic}[1]
\Require Observed data $\mathcal{D}=\{(X_i,Y_i,T_i)\}_{i=1}^n$
\Ensure Estimated potential outcomes $({\widehat{\mu}}_{0}(X_{i}),{\widehat{\mu}}_{1}(X_{i}))$

\State Initialize outcome-module parameters $\theta_Y$
\State Freeze treatment module parameters $\theta_T$.

\For{each training epoch}
    \For{each mini-batch $\mathcal{B}\subset \mathcal{D}$}
        \algstep{Step 1: Tokenization and Transformer Encoder}
        \State Obtain frozen treatment outputs:
        \[
        (\widehat e_i,H_i^{(T)})
        =
        \mathcal{F}_T(X_i;\theta_T).
        \]
        \State Apply cutting feedback:
        \[
        \widetilde H_i^{(T)}
        =
        \mathrm{stopgrad}(H_i^{(T)}),
        \qquad
        \widetilde e_i
        =
        \mathrm{stopgrad}(\widehat e_i).
        \]
        \State Construct outcome linear branch:
        \[
        H_i^{(L,Y)}=\mathrm{Tok}^{(L,Y)}(X_i).
        \]
        \State Construct outcome self-attention branch:
        \[
        H_i^{(S,Y)}=\mathrm{Enc}_Y\{\mathrm{Tok}^{(S,Y)}(X_i)\}.
        \]
        \algstep{Step 2: Arm-specific representation}
        \For{$t\in\{0,1\}$}
            \State Use arm-specific query $q^{(t)}$ to summarize outcome branches:
            \[
            z_i^{(L,t)}
            =
            \mathrm{MHA}(q^{(t)},H_i^{(L,Y)},H_i^{(L,Y)}),
            \]
            \[
            z_i^{(S,t)}
            =
            \mathrm{MHA}(q^{(t)},H_i^{(S,Y)},H_i^{(S,Y)}).
            \]
            \State Use one-way cross-attention from outcome to frozen treatment branch:
            \[
            z_i^{(T,t)}
            =
            \mathrm{MHA}(q^{(t)},\widetilde H_i^{(T)},\widetilde H_i^{(T)}).
            \]
            \algstep{Step 3: Fusion gate}
            \State Compute arm-specific fusion gate:
            \[
            \alpha_i^{(t)}
            =
            \mathrm{softmax}
            \left(
            \frac{
            g_i^{t}([z_i^{(L,t)};z_i^{(S,t)};z_i^{(T,t)}])
            }{\tau_g}
            \right).
            \]
            \State Form fused arm-specific representation:
            \[
            f_i^{(t)}
            =
            [
            \alpha_{iL}^{(t)}z_i^{(L,t)};
            \alpha_{iS}^{(t)}z_i^{(S,t)};
            \alpha_{iT}^{(t)}z_i^{(T,t)}
            ].
            \]
            \State Predict potential outcome:
            \[{\widehat{\mu}}_{t}(X_{i}) = W_{h2}^{(t)}\text{ }GELU(W_{h1}^{(t)}f_{i}^{(t)} + b_{h1}^{(t)}) + b_{h2}^{(t)},t \in \{ 0,1\}.\]
        \EndFor

        \State Update only $\theta_Y$ by minimizing factual outcome loss:
        \[
        \mathcal{L}_Y
        =
        \frac{1}{|\mathcal{B}_0|}
        \sum_{i\in\mathcal{B}_0}
        \{Y_i-\widehat\mu_0(X_i)\}^2
        +
        \frac{1}{|\mathcal{B}_1|}
        \sum_{i\in\mathcal{B}_1}
        \{Y_i-\widehat\mu_1(X_i)\}^2,
        \]
        \State where $\mathcal{B}_t=\{i\in\mathcal{B}:T_i=t\}$.
        \State Enforce no outcome-to-treatment feedback:
        \[
        \frac{\partial \mathcal{L}_Y}{\partial \theta_T}=0.
        \]
    \EndFor
\EndFor
\end{algorithmic}
\endgroup
\end{algorithm}
\newpage
\begin{algorithm}[H]
\caption{MOCA: Estimate treatment effects}
\label{alg:mocae}
\begingroup
\fontsize{10}{15}\selectfont
\renewcommand{\baselinestretch}{0.88}\selectfont

\setlength{\abovedisplayskip}{3pt}
\setlength{\belowdisplayskip}{3pt}
\setlength{\abovedisplayshortskip}{2pt}
\setlength{\belowdisplayshortskip}{2pt}

\begin{algorithmic}[1]
\Require Estimated potential outcomes $({\widehat{\mu}}_{0}(X_{i}),{\widehat{\mu}}_{1}(X_{i}))$
\Ensure Estimated CATE $\widehat{\tau}_i$ and ATE $\widehat{\tau}$

\For{$i=1,\ldots,n$}
    \State Estimate individual CATE:
    \[
    \widehat{\tau}_i
    =
    \widehat{\mu}_1(X_i)
    -
    \widehat{\mu}_0(X_i).
    \]
\EndFor

\State Estimate average treatment effect:
\[
\widehat{\tau}
=
\frac{1}{n}
\sum_{i=1}^n
\widehat{\tau}_i.
\]

\State \Return $\{\widehat{\tau}_i\}_{i=1}^n$ and $\widehat{\tau}$.

\end{algorithmic}
\endgroup
\end{algorithm}

\newpage
\section{Additional Tables and Figures}

\begin{table}[htbp]
\caption{Hyperparameter settings used in the simulation and real-data experiments.}
\label{tab:hyperparameters}
\centering
\resizebox{\textwidth}{!}{%
\begin{tabular}{lccccccc}
\toprule
Dataset &
Batch &
Epochs (T/O/NN) &
LR (T/O/NN) &
$\alpha$ &
$d_{\text{model}}$ &
Temp &
Trees \\
\midrule

Simulation &
128 &
40/60/100 &
1e-3/3e-4/1e-3 &
1 &
32 &
1 &
300 \\

Ablation &
128 &
40/60/-- &
1e-3/3e-4/-- &
1 &
32 &
1 &
-- \\

IHDP &
128 &
40/60/100 &
1e-3/3e-4/1e-3 &
1 &
32 &
1 &
300 \\

DW &
64 &
80/80/100 &
1e-3/3e-4/1e-3 &
1 &
32 &
1 &
300 \\
\bottomrule
\end{tabular}%
}
\vspace{0.2cm}

\begin{flushleft}
\footnotesize{
*Abbreviations: T = treatment module, O = outcome module, and NN = neural network baselines (including TARNet, DragonNet, and X-learner). LR denotes the corresponding learning rates. Trees denotes the number of trees used for BART and Causal Forest.
}
\end{flushleft}
\end{table}

\begingroup
\small
\setlength{\tabcolsep}{3pt}
\renewcommand{\arraystretch}{1.05}

\begin{longtable}{@{}lccc@{}}

\caption{Conformal inference results for individual treatment effects across simulation scenarios.}
\label{tab:ite-conformal-results}\\

\toprule
Method &
ITE Coverage &
Mean ITE Interval Length &
Median ITE Interval Length \\  
\midrule
\endfirsthead

\toprule
Method &
ITE Coverage &
Mean ITE Interval Length &
Median ITE Interval Length \\
\midrule
\endhead

\midrule
\multicolumn{4}{r}{Continued on next page} \\
\endfoot

\bottomrule
\endlastfoot

\multicolumn{4}{c}{\textbf{Linear}} \\
\midrule
X-Learner    & 0.923 & 10.799 & 11.779 \\
MOCA         & 0.957 & 11.321 & 12.885 \\
TARNet       & 0.912 & 7.008  & 7.335  \\
DragonNet    & 0.914 & 6.997  & 7.163  \\
BART         & 0.923 & 7.674  & 8.343  \\
CausalForest & 0.970 & 12.518 & 14.169 \\
Do\_PFN      & 0.866 & 7.390  & 7.572  \\

\midrule
\multicolumn{4}{c}{\textbf{Linear-Omis}} \\
\midrule
X-Learner    & 0.941 & 12.691 & 13.144 \\
MOCA         & 0.963 & 12.591 & 13.007 \\
TARNet       & 0.921 & 8.758  & 8.967  \\
DragonNet    & 0.919 & 8.678  & 8.808  \\
BART         & 0.947 & 8.876  & 8.833  \\
CausalForest & 0.970 & 13.352 & 13.895 \\
Do\_PFN      & 0.900 & 9.096  & 9.119  \\

\midrule
\multicolumn{4}{c}{\textbf{Non-linear}} \\
\midrule
X-Learner    & 0.953 & 11.159 & 10.059 \\
MOCA         & 0.980 & 13.734 & 12.442 \\
TARNet       & 0.929 & 7.672  & 7.207  \\
DragonNet    & 0.933 & 7.710  & 7.293  \\
BART         & 0.960 & 8.625  & 8.268  \\
CausalForest & 0.969 & 11.306 & 10.173 \\
Do\_PFN      & 0.960 & 13.266 & 12.115 \\

\midrule
\multicolumn{4}{c}{\textbf{\textit{t} Distribution}} \\
\midrule
X-Learner    & 0.943 & 20.852 & 25.674 \\
MOCA         & 0.952 & 21.138 & 26.921 \\
TARNet       & 0.898 & 9.154  & 10.234 \\
DragonNet    & 0.900 & 9.366  & 10.309 \\
BART         & 0.942 & 17.637 & 21.510 \\
CausalForest & 0.962 & 23.670 & 29.987 \\
Do\_PFN      & 0.876 & 18.123 & 18.426 \\

\midrule
\multicolumn{4}{c}{\textbf{Hidden Confounding}} \\
\midrule
X-Learner    & 0.912 & 16.639 & 17.620 \\
MOCA         & 0.952 & 17.349 & 19.355 \\
TARNet       & 0.890 & 12.632 & 13.436 \\
DragonNet    & 0.892 & 12.483 & 13.261 \\
BART         & 0.900 & 13.195 & 14.342 \\
CausalForest & 0.945 & 16.752 & 18.330 \\
Do\_PFN      & 0.922 & 14.147 & 14.841 \\

\midrule
\multicolumn{4}{c}{\textbf{Hidden Confounding MultiU}} \\
\midrule
X-Learner    & 0.896 & 17.430 & 17.426 \\
MOCA         & 0.940 & 18.283 & 18.919 \\
TARNet       & 0.897 & 14.360 & 14.295 \\
DragonNet    & 0.896 & 14.270 & 14.271 \\
BART         & 0.883 & 14.549 & 14.434 \\
CausalForest & 0.927 & 17.235 & 17.645 \\
Do\_PFN      & 0.927 & 15.375 & 15.390 \\

\midrule
\multicolumn{4}{c}{\textbf{High Dimension}} \\
\midrule
X-Learner    & 0.927 & 23.010 & 27.512 \\
MOCA         & 0.979 & 22.666 & 28.654 \\
TARNet       & 0.906 & 18.480 & 22.154 \\
DragonNet    & 0.903 & 18.527 & 22.162 \\
BART         & 0.886 & 19.159 & 23.289 \\
CausalForest & 0.939 & 19.952 & 24.466 \\
Do\_PFN      & 0.984 & 22.832 & 26.408 \\

\end{longtable}

\endgroup
\newpage
\setcounter{figure}{0}
\renewcommand{\thefigure}{S\arabic{figure}}

\begin{figure}[htbp]
    \centering
    \includegraphics[width=0.95\linewidth,trim=0 2 1 0,clip]{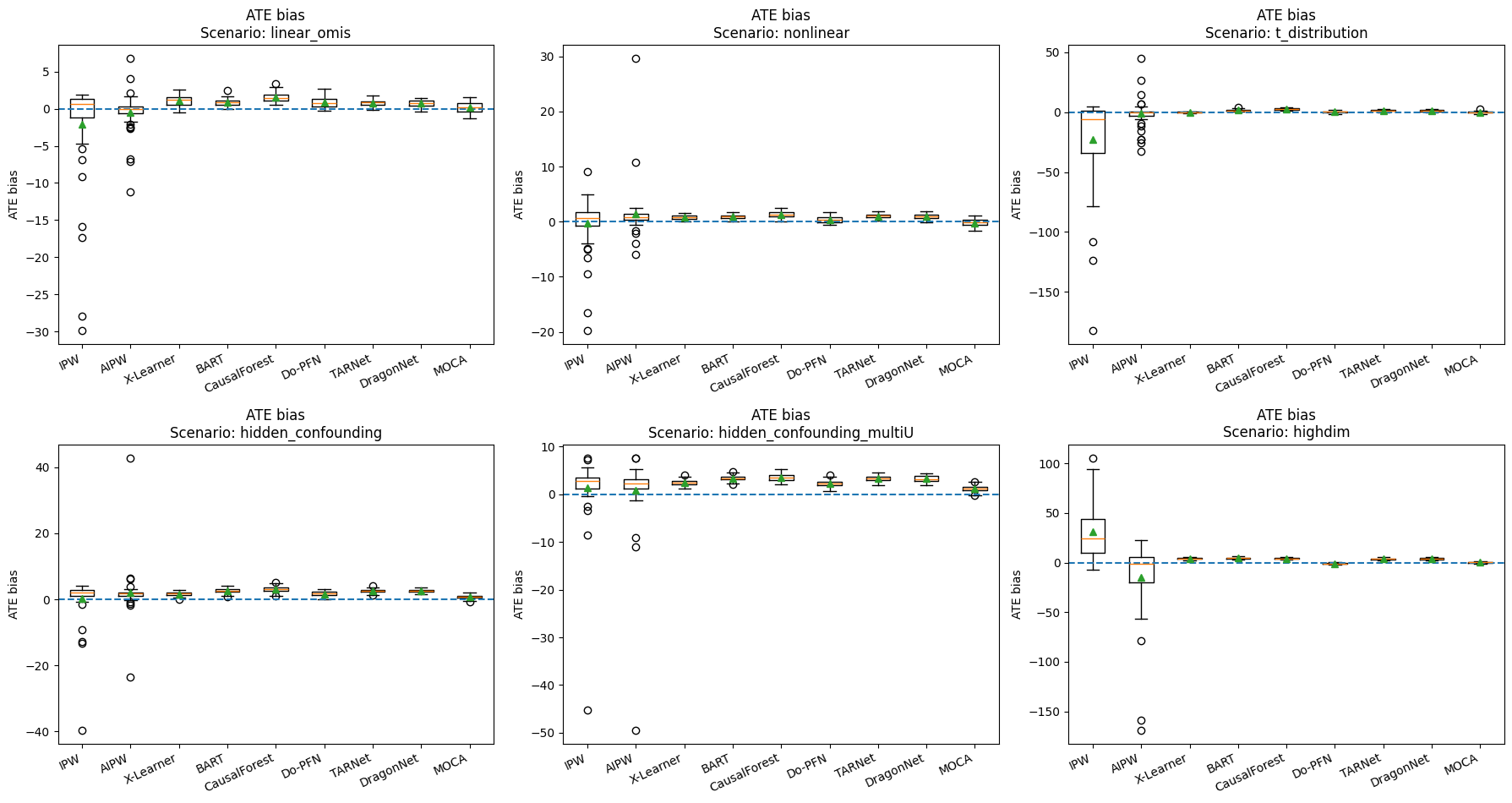}
    \caption{ATE bias boxplots for all simulation scenarios.}\label{fig:ate-bias}
\end{figure}

\begin{figure}[H]
    \centering
    \includegraphics[width=0.95\linewidth,trim=0 2 1 0,clip]{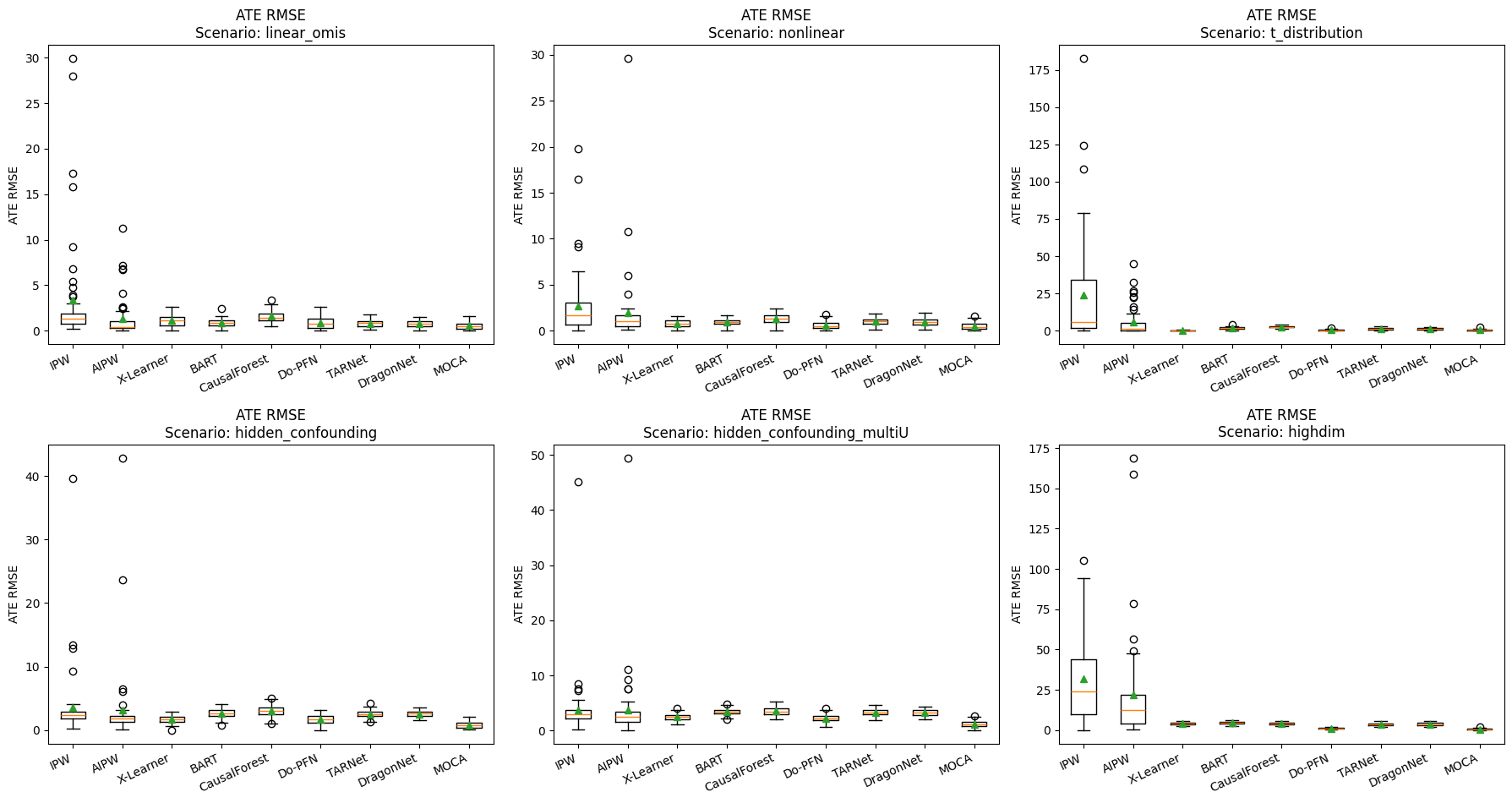}
    \caption{ATE RMSE boxplots for all simulation scenarios.}\label{fig:ate-rmse}
\end{figure}

\begin{figure}[H]
    \centering
    \includegraphics[width=0.95\linewidth,trim=0 2 1 0,clip]{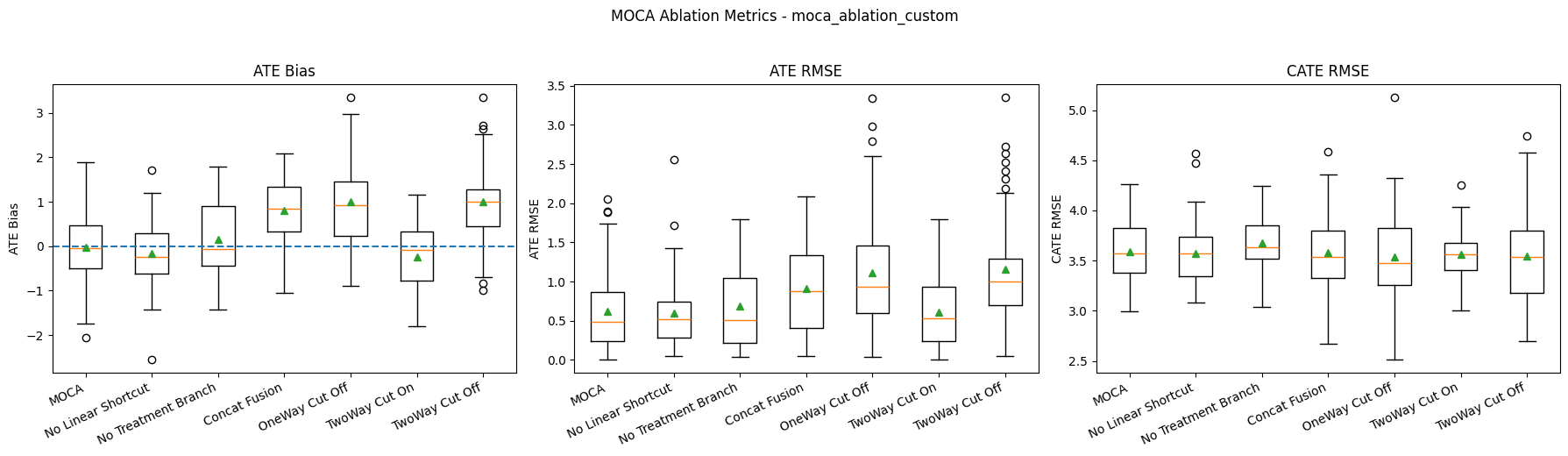}
    \caption{Summary boxplots for the ablation study.}\label{fig:ablation}
\end{figure}

\newpage

\clearpage
\begingroup
\small
\setlength{\bibsep}{2pt}
\linespread{0.92}\selectfont
\bibliographystyle{plainnat}
\bibliography{moca_refs_merged}

@article{rosenbaum1983,
  author = {Rosenbaum, Paul R. and Rubin, Donald B.},
  title = {The Central Role of the Propensity Score in Observational Studies for Causal Effects},
  journal = {Biometrika},
  year = {1983},
  volume = {70},
  number = {1},
  pages = {41--55}
}

@article{fritz2020synthetic,
  author  = {Fritz, Tobias},
  title   = {A Synthetic Approach to Markov Kernels, Conditional Independence and Theorems on Sufficient Statistics},
  journal = {Advances in Mathematics},
  volume  = {370},
  pages   = {107239},
  year    = {2020},
  doi     = {10.1016/j.aim.2020.107239}
}

@article{janzing2010causal,
  author  = {Janzing, Dominik and Sch{\"o}lkopf, Bernhard},
  title   = {Causal Inference Using the Algorithmic Markov Condition},
  journal = {IEEE Transactions on Information Theory},
  volume  = {56},
  number  = {10},
  pages   = {5168--5194},
  year    = {2010},
  month   = {October},
  doi     = {10.1109/TIT.2010.2060095}
}

@book{Kallenberg1997,
  author    = {Kallenberg, Olav},
  title     = {Foundations of Modern Probability},
  year      = {1997},
  edition   = {1},
  series    = {Probability and Its Applications},
  address   = {New York, NY},
  publisher = {Springer},
  doi       = {10.1007/b98838}
}

@article{mccandless2010,
  author  = {McCandless, Lawrence C. and Douglas, Ian J. and Evans, Stephen J. and Smeeth, Liam},
  title   = {Cutting Feedback in Bayesian Regression Adjustment for the Propensity Score},
  journal = {The International Journal of Biostatistics},
  year    = {2010},
  volume  = {6},
  number  = {2},
  note    = {Article 16},
  doi     = {10.2202/1557-4679.1205}
}

@article{zigler2013,
  author  = {Zigler, Corwin M. and Watts, Krista and Yeh, Robert W. and Wang, Yun and Coull, Brent A. and Dominici, Francesca},
  title   = {Model Feedback in Bayesian Propensity Score Estimation},
  journal = {Biometrics},
  year    = {2013},
  volume  = {69},
  number  = {1},
  pages   = {263--273},
  doi     = {10.1111/j.1541-0420.2012.01830.x}
}

@article{rubin2007,
  author  = {Rubin, Donald B.},
  title   = {The Design versus the Analysis of Observational Studies for Causal Effects: Parallels with the Design of Randomized Trials},
  journal = {Statistics in Medicine},
  year    = {2007},
  volume  = {26},
  number  = {1},
  pages   = {20--36},
  doi     = {10.1002/sim.2739}
}

@inproceedings{VowelsEtAl2021,
  author    = {Vowels, Matthew and Camgoz, Necati Cihan and Bowden, Richard},
  title     = {Targeted VAE: Variational and Targeted Learning for Causal Inference},
  booktitle = {2021 IEEE International Conference on Smart Data Services},
  pages     = {132--141},
  year      = {2021},
  publisher = {IEEE},
  doi       = {10.1109/SMDS53860.2021.00027}
}

@book{ImbensRubin2015,
author    = {Imbens, Guido W. and Rubin, Donald B.},
title     = {Causal Inference in Statistics, Social, and Biomedical Sciences},
year      = {2015},
address   = {Cambridge},
publisher = {Cambridge University Press}
}

@article{barnard1999,
  author  = {Barnard, John and Rubin, Donald B.},
  title   = {Small-sample degrees of freedom with multiple imputation},
  journal = {Biometrika},
  volume  = {86},
  number  = {4},
  pages   = {948--955},
  year    = {1999}
}

@article{dehejia1999,
  author = {Dehejia, Rajeev H. and Wahba, Sadek},
  title = {Causal Effects in Nonexperimental Studies: Reevaluating the Evaluation of Training Programs},
  journal = {Journal of the American Statistical Association},
  year = {1999},
  volume = {94},
  number = {448},
  pages = {1053--1062}
}

@article{robins2000,
  author = {Robins, James M. and Hernan, Miguel Angel and Brumback, Babette},
  title = {Marginal Structural Models and Causal Inference in Epidemiology},
  journal = {Epidemiology},
  year = {2000},
  volume = {11},
  number = {5},
  pages = {550--560}
}

@article{bang2005,
  author = {Bang, Heejung and Robins, James M.},
  title = {Doubly Robust Estimation in Missing Data and Causal Inference Models},
  journal = {Biometrics},
  year = {2005},
  volume = {61},
  number = {4},
  pages = {962--973}
}

@article{rubin2005,
  author = {Rubin, Donald B.},
  title = {Causal Inference Using Potential Outcomes: Design, Modeling, Decisions},
  journal = {Journal of the American Statistical Association},
  year = {2005},
  volume = {100},
  number = {469},
  pages = {322--331}
}

@article{bayarri2009,
  author = {Bayarri, Marıa Jesús and Berger, James O. and Liu, Fei},
  title = {Modularization in Bayesian Analysis, with Emphasis on Analysis of Computer Models},
  journal = {Bayesian Analysis},
  year = {2009}
}

@book{hastie2009,
  author = {Hastie, Trevor and Tibshirani, Robert and Friedman, Jerome},
  title = {The Elements of Statistical Learning: Data Mining, Inference, and Prediction},
  publisher = {Springer},
  year = {2009},
  edition = {2}
}

@article{vanderweele2011,
  author = {VanderWeele, Tyler J. and Arah, Onyebuchi A.},
  title = {Unmeasured Confounding for General Outcomes, Treatments, and Confounders: Bias Formulas for Sensitivity Analysis},
  journal = {Epidemiology},
  year = {2011},
  volume = {22},
  number = {1},
  pages = {42--52}
}

@inproceedings{johansson2016,
  author = {Johansson, Fredrik and Shalit, Uri and Sontag, David},
  title = {Learning Representations for Counterfactual Inference},
  booktitle = {Proceedings of the 33rd International Conference on Machine Learning},
  year = {2016},
  pages = {3020--3029}
}

@article{louizos2017,
  author = {Louizos, Christos and Shalit, Uri and Mooij, Joris and Sontag, David and Zemel, Richard and Welling, Max},
  title = {Causal Effect Inference with Deep Latent-Variable Models},
  journal = {Advances in Neural Information Processing Systems},
  year = {2017},
  volume = {30}
}

@article{vaswani2017,
  author = {Vaswani, Ashish and Shazeer, Noam and Parmar, Niki and Uszkoreit, Jakob and Jones, Llion and Gomez, Aidan N. and Kaiser, {
Lukasz} and Polosukhin, Illia},
  title = {Attention Is All You Need},
  journal = {Advances in Neural Information Processing Systems},
  year = {2017},
  volume = {30}
}

@article{iacus2019,
  author = {Iacus, Stefano M. and King, Gary and Porro, Giuseppe},
  title = {A Theory of Statistical Inference for Matching Methods in Causal Research},
  journal = {Political Analysis},
  year = {2019},
  volume = {27},
  number = {1},
  pages = {46--68}
}

@article{kunzel2019,
  author = {K{\"u}nzel, S{\"o}ren R. and Sekhon, Jasjeet S. and Bickel, Peter J. and Yu, Bin},
  title = {Metalearners for Estimating Heterogeneous Treatment Effects Using Machine Learning},
  journal = {Proceedings of the National Academy of Sciences},
  year = {2019},
  volume = {116},
  number = {10},
  pages = {4156--4165}
}

@inproceedings{lee2019,
  author = {Lee, Juho and Lee, Yoonho and Kim, Jungtaek and Kosiorek, Adam and Choi, Seungjin and Teh, Yee Whye},
  title = {Set Transformer: A Framework for Attention-Based Permutation-Invariant Neural Networks},
  booktitle = {Proceedings of the 36th International Conference on Machine Learning},
  year = {2019},
  pages = {3744--3753}
}

@article{li2019,
  author = {Li, Fan and Thomas, Laine E. and Li, Fan},
  title = {Addressing Extreme Propensity Scores via the Overlap Weights},
  journal = {American Journal of Epidemiology},
  year = {2019},
  volume = {188},
  number = {1},
  pages = {250--257}
}

@article{shi2019,
  author = {Shi, Claudia and Blei, David and Veitch, Victor},
  title = {Adapting Neural Networks for the Estimation of Treatment Effects},
  journal = {Advances in Neural Information Processing Systems},
  year = {2019},
  volume = {32}
}

@article{lei2021conformal,
  title={Conformal Inference of Counterfactuals and Individual Treatment Effects},
  author={Lei, Lihua and Cand{\`e}s, Emmanuel J.},
  journal={Journal of the Royal Statistical Society: Series B (Statistical Methodology)},
  volume={83},
  number={5},
  pages={911--938},
  year={2021},
  publisher={Oxford University Press}
}

@article{fan2021,
  author = {Fan, Jianqing and Wang, Weichen and Zhu, Ziwei},
  title = {A Shrinkage Principle for Heavy-Tailed Data: High-Dimensional Robust Low-Rank Matrix Recovery},
  journal = {The Annals of Statistics},
  year = {2021},
  volume = {49},
  number = {3},
  pages = {1239--1266}
}

@inproceedings{melnychuk2022,
  author = {Melnychuk, Valentyn and Frauen, Dennis and Feuerriegel, Stefan},
  title = {Causal Transformer for Estimating Counterfactual Outcomes},
  booktitle = {Proceedings of the 39th International Conference on Machine Learning},
  year = {2022},
  pages = {15293--15329}
}

@article{babiloni2023,
  author = {Babiloni, Francesca and Marras, Ioannis and Deng, Jiankang and Kokkinos, Filippos and Maggioni, Matteo and Chrysos, Grigorios and Torr, Philip and Zafeiriou, Stefanos},
  title = {Linear Complexity Self-Attention with 3rd Order Polynomials},
  journal = {IEEE Transactions on Pattern Analysis and Machine Intelligence},
  year = {2023},
  volume = {45},
  number = {11},
  pages = {12726--12737}
}

@article{frazier2025,
  author  = {Frazier, David T. and Nott, David J.},
  title   = {Cutting Feedback and Modularized Analyses in Generalized Bayesian Inference},
  journal = {Bayesian Analysis},
  year    = {2025},
  volume  = {20},
  number  = {4},
  pages   = {1647--1675}
}

@article{liu2025,
  author  = {Liu, Yang and Goudie, Robert JB},
  title   = {A General Framework for Cutting Feedback within Modularized Bayesian Inference},
  journal = {Journal of the Royal Statistical Society Series B: Statistical Methodology},
  year    = {2025},
  volume  = {87},
  number  = {4},
  pages   = {1171--1199}
}

@article{wang2025,
  author = {Wang, Aiwen and Piao, Xinglin and Zhang, Xinze and Guo, Yujia and Zhang, Yong},
  title = {RAT: Residual Attention Transformer for Tabular Data},
  journal = {IEEE Transactions on Big Data},
  year = {2025}
}

@article{chipman2010bart,
  title={BART: Bayesian Additive Regression Trees},
  author={Chipman, Hugh A. and George, Edward I. and McCulloch, Robert E.},
  journal={The Annals of Applied Statistics},
  volume={4},
  number={1},
  pages={266--298},
  year={2010}
}

@article{wager2018causalforest,
  title={Estimation and Inference of Heterogeneous Treatment Effects using Random Forests},
  author={Wager, Stefan and Athey, Susan},
  journal={Journal of the American Statistical Association},
  volume={113},
  number={523},
  pages={1228--1242},
  year={2018}
}

@article{robertson2026dopfn,
  title={Do-PFN: In-Context Learning for Causal Effect Estimation},
  author={Robertson, Jake and Reuter, Arik and Guo, Siyuan and Hollmann, Noah and Hutter, Frank and Sch{\"o}lkopf, Bernhard},
  journal={Advances in Neural Information Processing Systems},
  volume={38},
  pages={174811--174848},
  year={2025}
}
\endgroup

\end{document}